%% file: main.tex
\documentclass{article}
\usepackage{iclr2027_conference,times}
\usepackage[T1]{fontenc}
\usepackage{microtype,graphicx,booktabs,array}
\usepackage{amsmath,amssymb,amsthm,mathtools}
\usepackage{algorithm,algpseudocode,placeins}
\usepackage{wrapfig,needspace,float,xcolor,hyperref,zref-savepos}
\definecolor{paper}{HTML}{FAF6EF}
\definecolor{ink}{HTML}{2B2D42}
\definecolor{method}{HTML}{7B2FBE}
\definecolor{referenceblue}{HTML}{4A6FA5}
\hypersetup{colorlinks=true,allcolors=referenceblue,pdftitle={Cost-free Spectral Estimation for Adaptive Newton--Schulz in Matrix Optimizers}}
\usepackage[capitalize,noabbrev]{cleveref}
\usepackage{etoolbox}
\AtBeginEnvironment{table}{\setlength{\belowcaptionskip}{4pt}}
\graphicspath{{figures/}}
\newtheorem{theorem}{Theorem}
\newtheorem{proposition}{Proposition}

\theoremstyle{remark}
\DeclareMathOperator{\tr}{tr}
\DeclareMathOperator{\polar}{polar}
\DeclareMathOperator{\diag}{diag}
\DeclareMathOperator*{\argmin}{arg\,min}
\DeclareMathOperator*{\argmax}{arg\,max}
\newcommand{\R}{\mathbb R}
\newcommand{\eps}{\varepsilon}
\newcommand{\dd}{\,\mathrm d}
\newcommand{\wh}{\widehat w}
\newcommand{\yh}{\widehat y}
\newcommand{\Eh}{\widehat E}
\newcommand{\region}{\mathcal R}
\newcommand{\PE}{\mathrm{PE}}
\newcommand{\best}[1]{\boldsymbol{#1}}
\newlength{\wrapclearance}
\newcommand{\clearrefbox}{\par
  \setlength{\wrapclearance}{\dimexpr\zposy{nstextend}sp-\zposy{nsboxend}sp\relax}%
  \ifdim\wrapclearance>0pt\vspace*{\wrapclearance}\fi}
\title{Cost-free Spectral Estimation for Adaptive Newton--Schulz in Matrix Optimizers}

\author{Kristi Topollai \\
New York University\\
\texttt{kt2664@nyu.edu} \And
  Anna Choromanska \\
  New York University \\
  \texttt{ac5455@nyu.edu} \\
}

\iclrfinalcopy
\begin{document}
\maketitle
\lhead{}                            

\begin{abstract}
Matrix optimizers such as Muon transform each momentum matrix through an approximate orthogonalization, typically implemented by a small number of Newton--Schulz matrix multiplications.  The quality and cost of this approximation depend strongly on the singular-value spectrum of its input, yet existing implementations use the same fixed polynomial routine for every layer and throughout training.  We show that this uniform treatment is unnecessary: the computations in the Newton--Schulz method already reveal enough information to make the method adaptive.  The Gram matrices formed inside Newton--Schulz iterations yield spectral moments through inexpensive scalar reductions, requiring no additional matrix multiplications.  From these moments, we recover an estimate of the empirical singular-value distribution and use it to select a polynomial routine specialized to the current matrix.  This turns Newton--Schulz orthogonalization into a spectrum-adaptive procedure that responds to differences across both layers and training time.  On saved momentum matrices, spectral estimation substantially reduces orthogonalization error at a fixed iteration budget or reaches the same accuracy with fewer iterations, and in GPT pretraining up to 1B parameters it lowers the validation loss of two matrix optimizers.  Our results suggest that matrix-function operations inside optimizers need not be designed for a conservative worst-case spectrum: they can cheaply measure the spectrum they are already processing and specialize computations accordingly.
\end{abstract}

\section{Introduction}
\label{sec:intro}

Matrix optimizers are an increasingly practical alternative to elementwise
optimization for neural-network training. Methods such as Shampoo
\citep{gupta2018shampoo,anil2020scalable}, SOAP \citep{vyas2025soap},
Dion \citep{ahn2025dion}, and Muon
\citep{jordan2024muon,liu2025muon} exploit matrix structure in gradients,
momentum, or preconditioning statistics rather than treating every parameter
independently. Muon has attracted substantial attention for
language-model pretraining, where it approximately orthogonalizes momentum
matrices before using them as updates. Recent large-scale systems have adopted
Muon or closely related matrix-optimization ideas, including DeepSeek-V4, GLM-5, and
Kimi K3 \citep{deepseek2026v4,glm5team2026glm5,kimi2026k3}, and practical studies show
that Muon is competitive at language-model scale
\citep{essential2025practical,liu2025muon}.

\begin{figure}[t]
\centering
\includegraphics[width=\textwidth]{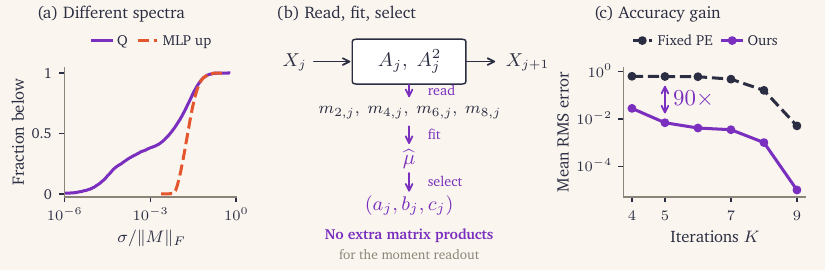}
\caption{\textbf{Newton--Schulz can measure its input and specialize its own computation.}
\textbf{(a)} Full singular-value CDFs show different momentum spectra for
attention Q and MLP-up at layer $5$, step $16$k.
\textbf{(b)} Read moments from existing matrix products, fit the spectrum, and
select a routine for subsequent steps.
\textbf{(c)} On $540$ saved MLP matrices, spectral estimation reduces mean
polar RMS error by $90\times$ relative to fixed Polar Express (PE) at $K=5$ iterations, scored above $10^{-5}$ in
Frobenius-normalized coordinates.}
\label{fig:overview}
\end{figure}

This growing use of matrix optimizers makes the cost of their internal matrix
functions increasingly important. Muon approximately computes the polar factor
of each momentum matrix with a small number of Newton--Schulz iterations, replacing an SVD
with hardware-efficient matrix multiplications
\citep{higham2008functions,amsel2026polar}. Online KL Shampoo (OKLS)
\citep{okls2026} similarly computes the inverse square roots of its
preconditioner statistics with a coupled Newton--Schulz iteration. Recent work improves these
iterations through better polynomial design and scaling
\citep{cans2025,amsel2026polar}, but typically applies the same routine across
matrices. This can be wasteful because the difficulty of the approximation
depends strongly on the input spectrum, and momentum spectra vary across matrix
types, layers, and training time, with attention often substantially broader
than MLPs. A conservative routine that handles the hardest cases can spend unnecessary iterations on easier ones.

Several recent methods have begun to exploit this heterogeneity. AMO adapts
orthogonalization at the matrix-type level using an observe-then-commit
strategy \citep{amo2026}, while PRISM adapts matrix-function computation using
randomized spectral information during the iterations \citep{prism2026}. \Cref{app:related} reviews related work. The best matrix-function routine depends on the spectrum it is applied to, so
exploiting this during training requires spectral information that is cheap to
obtain.
Newton--Schulz already contains such information. Its intermediate matrices
encode measurements of the input singular-value distribution that can be
extracted through inexpensive scalar reductions, without additional matrix
multiplications. We use them to estimate enough of the spectrum to select a
better-matched routine, a procedure we call \emph{spectral estimation}
(\cref{fig:overview}).

We make three contributions.
\begin{itemize}
    \item We show that the intermediate matrices of Newton--Schulz yield
    spectral measurements at negligible cost. They need no additional matrix
    multiplications, and fitting a spectrum to them costs the same at every
    matrix size, unlike an SVD.

    \item We use these measurements to build a spectrum-adaptive
    Newton--Schulz procedure that selects a polynomial routine for each matrix
    instead of one conservative choice.

    \item We evaluate spectral estimation both offline on saved language-model
    momentum matrices and end-to-end in GPT pretraining. Offline, it
    substantially improves polar approximation at a fixed iteration budget or
    reaches the same accuracy with fewer Newton--Schulz iterations. In pretraining,
    it improves two different matrix optimizers, Muon and OKLS.
\end{itemize}

\section{Matrix functions in optimization}
\label{sec:preliminaries}
\subsection{From matrix updates to spectral transformations}

Matrix optimizers act on gradients, momentum, or preconditioning statistics through matrix-valued transformations. Examples include Muon \citep{jordan2024muon,liu2025muon}, Dion \citep{ahn2025dion}, AdaMuon \citep{si2025adamuon}, and Shampoo-style methods \citep{gupta2018shampoo,anil2020scalable,vyas2025soap}. Muon is our main example. For a momentum matrix $M=U\Sigma V^\top$, its ideal update is the polar factor $\polar(M)=UV^\top$, which preserves singular-vector directions while mapping every nonzero singular value to one \citep{jordan2024muon}.

Other matrix optimizers apply different spectral transformations. Shampoo and SOAP construct dense preconditioners from matrix-valued second-moment statistics, typically using eigendecomposition or related dense linear algebra \citep{gupta2018shampoo,anil2020scalable,vyas2025soap}. OKLS instead computes inverse square roots of those statistics using a coupled Newton--Schulz iteration \citep{okls2026}. Both orthogonalization and dense preconditioning reduce to evaluating matrix functions such as the polar factor, inverse roots, or related spectral powers. Direct SVDs or eigendecompositions make these functions easy to define but expensive to evaluate repeatedly during training. Newton--Schulz and related polynomial iterations replace them with short sequences of matrix multiplications that map efficiently to GPUs \citep{higham2008functions,amsel2026polar}. We focus on polar approximation, while coupled variants provide analogous routes to inverse roots and fractional powers \citep{higham1997stable,qi2026delving}.

\subsection{Newton--Schulz and optimal polynomial iterations}

Newton--Schulz orthogonalization applies a sequence of low-degree polynomials to the singular values of a matrix. The classical update
$x\mapsto(3x-x^3)/2$ drives normalized positive singular values toward one
\citep{bjorck1971iterative,kovarik1970}. Modern variants accelerate this process through improved scaling and higher-degree polynomial iterations
\citep{chen2014scaling,amsel2026polar,cans2025}. Orient $M\in\R^{r\times n}$ so that $r\le n$, and let $X_0$ be a rescaling whose singular values $\gamma_i$ lie in $[0,1]$. A quintic iteration takes the form
\begin{equation}
A_j=X_jX_j^\top,\qquad
X_{j+1}=a_jX_j+(b_jA_j+c_jA_j^2)X_j.
\label{eq:loop}
\end{equation}
If $p_j(x)=a_jx+b_jx^3+c_jx^5$, then after $K$ iterations,
\begin{equation}
X_K
=
U\diag(P(\gamma_i))V^\top,
\qquad
P=p_{K-1}\circ\cdots\circ p_0.
\label{eq:scalar}
\end{equation}
The matrix iteration is therefore determined by its scalar response $P$, and the coefficients set how quickly singular values move toward one without changing an iteration's matrix-multiplication pattern.

Polar Express (PE) chooses these coefficients by minimizing worst-case approximation error over a prescribed spectral interval \citep{amsel2026polar}. For an iteration with design interval $[\ell_j,u_j]$,
\begin{equation}
(a_j^\star,b_j^\star,c_j^\star)
\in
\argmin_{a,b,c\in\R}
\max_{x\in[\ell_j,u_j]}
|1-(ax+bx^3+cx^5)|.
\label{eq:minimax-stage}
\end{equation}
If the resulting error is $e_j$, the next iteration is designed for the output interval $[1-e_j,1+e_j]$. This greedy construction is optimal among compositions of the same depth and degree \citep[Theorem~3.1]{amsel2026polar}, and CANS similarly uses Chebyshev-type design \citep{cans2025}.

This optimality is relative to the chosen spectral interval. For an initial design interval $[\ell_0,u_0]$, rescaling the upper endpoint to one leaves the ratio $\rho=\ell_0/u_0$ as the only relevant quantity. A broader interval therefore defines a harder problem.

\section{Spectrum-dependent cost of Newton--Schulz}
\label{sec:heterogeneity}
\subsection{Broader design intervals require more computation}

The following result shows how widening the design interval degrades the best error at fixed depth.

\begin{theorem}[Interval width and minimax error]
\label{thm:interval-cost}
Let $\eps_K^\star(\rho)$ be the minimum uniform error on $[\rho,1]$ over all compositions of $K$ odd polynomials of degree at most five, and write $e(\rho)=\eps_1^\star(\rho)$ for the one-iteration error. In exact arithmetic, ideal PE attains this minimum, and for $K\ge1$
\begin{equation}
    \eps_{K+1}^\star(\rho)
    =
    e\!\left(
    \frac{1-\eps_K^\star(\rho)}
         {1+\eps_K^\star(\rho)}
    \right).
    \label{eq:interval-recursion}
\end{equation}
For $K\ge1$ and $0<\rho_1<\rho_2<1$,
\begin{equation}
    \eps_K^\star(\rho_1)>\eps_K^\star(\rho_2),\qquad
    \eps_K^\star(\rho)\ge
    \left(\frac{1-\rho}{1+\rho}\right)^{5^K},
    \quad 0<\rho<1.
    \label{eq:interval-lower}
\end{equation}
\end{theorem}

By \cref{thm:interval-cost}, a broader interval gives strictly larger minimax error at every fixed depth. As $\rho\to0$, the error approaches one, so achieving a fixed target accuracy eventually requires more iterations. The proof and corresponding depth bound are in \cref{app:interval}. \Cref{fig:interval} shows the effect directly. Broader intervals leave more error at the same depth because the polynomial must amplify small singular values while remaining bounded on the upper end. Narrower ones reach low error sooner.

\begin{figure}[!htbp]
    \centering
    \includegraphics[width=\textwidth]{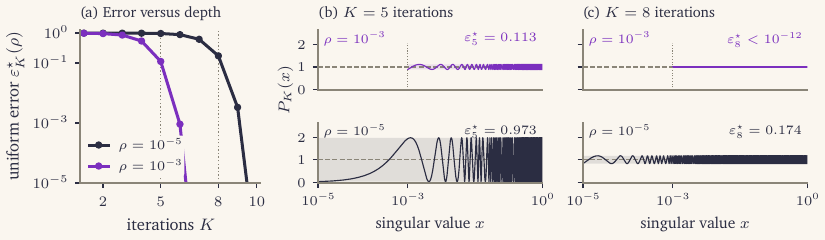}
    \caption{\textbf{Wider design intervals require more iterations.}
    Ideal PE on $[\rho,1]$.
    \textbf{(a)} Optimal uniform error versus depth.
    \textbf{(b,c)} Composed responses at $K=5$ and $K=8$.
    Shading marks $1\pm\eps_K^\star(\rho)$.}
    \label{fig:interval}
\end{figure}

\subsection{Training produces different spectra across layers and time}

Orthogonalization difficulty varies across matrix types, layers, and training time \citep{amo2026}. \Cref{fig:spectra} shows this for attention-query and MLP-up matrices of the $160$M Muon baseline of our pretraining experiments. Attention spectra are broader and MLP spectra more concentrated, and both vary across layers and training. A single conservative interval must cover all of these cases, even where it contains little spectral mass for a given matrix.

\begin{figure}[!htbp]
\centering
\includegraphics[width=\textwidth]{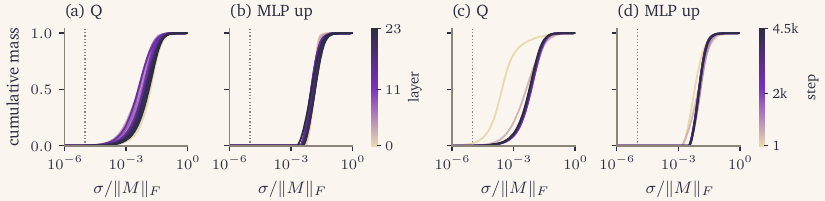}
\caption{\textbf{Spectra vary across layers and training.}
CDFs of the $160$M Muon baseline.
\textbf{(a,b)} Q and MLP-up across layers at step $2$k.
\textbf{(c,d)} Layer $11$ across checkpoints. The dotted line marks $10^{-5}$.}
\label{fig:spectra}
\end{figure}

\subsection{Spectral information identifies unnecessary work}

An oracle that knows the exact nonzero spectral range can design PE for
$[\sigma_{\min},\sigma_{\max}]$.
\Cref{fig:motivation} compares this with a common conservative interval.
Attention still requires a relatively deep iteration, but the narrower MLP spectrum can reach the same target in far fewer iterations.
\begin{figure}[!htbp]
\centering
\includegraphics[width=\textwidth]{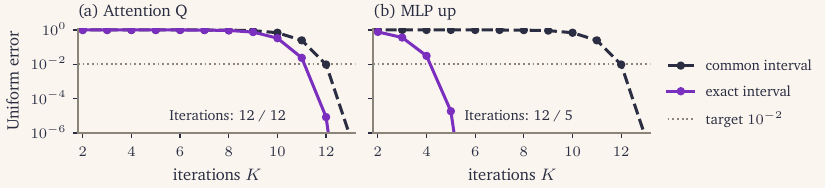}
\caption{\textbf{A conservative interval spends unnecessary work on narrow spectra.}
Q and MLP-up at layer $5$, step $16$k.
The common interval is $[10^{-7},1]$, and the exact interval is each matrix's nonzero spectral range, as used by an oracle.
Labels show the first depth reaching $10^{-2}$.}
\label{fig:motivation}
\end{figure}
Endpoints determine the best uniform guarantee, but average error depends on where the spectral mass lies. For normalized nonzero singular values $\gamma_i$, define
\begin{equation}
E_{\mathrm{RMS}}(P)^2
=
\frac{1}{r_+}
\sum_{i=1}^{r_+}
|1-P(\gamma_i)|^2,
\qquad
r_+=\operatorname{rank}(M).
\label{eq:spectral-rms}
\end{equation}
This is the squared relative Frobenius error to the polar factor. Minimizing worst-case error need not minimize this average, since allowing larger error in a sparse tail can improve accuracy over the bulk \citep{fan2020spectrum}. We therefore use spectral information beyond the endpoints to score candidate routines.
Obtaining this information is usually expensive. An exact SVD of $M\in\R^{r\times n}$ costs $O(r^2n)$ in sequential steps that map poorly to GPUs, and this cost grows with the matrix size. Newton--Schulz provides this information without extra matrix products, at a cost that stays nearly constant as matrices grow (\cref{sec:measure}).

\section{Spectral estimation from Newton--Schulz}
\label{sec:measure}
Choosing the best routine requires information about the input spectrum, and
computing it with an SVD would defeat the purpose. If an SVD were affordable at
every step, we could compute the polar factor directly. Instead, Newton--Schulz
itself lets us estimate the spectrum cheaply. We collect the measurements during
a measurement pass and reuse the chosen routine over subsequent steps. For nonzero $M$, let $x_i=\sigma_i(M)/\|M\|_F$ and
$X_0=M/(s\|M\|_F)$, where $s$ upper-bounds $\max_i x_i$.
Then $\gamma_i=x_i/s\in[0,1]$ and $\mu=\frac1r\sum_{i=1}^r\delta_{\gamma_i}$ is the empirical scaled spectrum, including zeros.

\subsection{Moments from the Newton--Schulz computation}

\begin{wrapfigure}{l}{0.48\textwidth}
\vspace{-10pt}
\setlength{\fboxsep}{6pt}
\colorbox{paper}{\begin{minipage}{\dimexpr\linewidth-12pt\relax}
\footnotesize
\textbf{Measurement pass}\par

\begin{algorithmic}[1]
\Require $X_0$, coefficients $(a_j,b_j,c_j)$
\For{$j=0,\ldots,K_{\mathrm{meas}}-1$}
\State $A_j\gets X_jX_j^\top$ \hfill\textcolor{referenceblue}{(1)}
\State $B_j\gets A_jA_j$ \hfill\textcolor{referenceblue}{(2)}
\State \textcolor{method}{\mbox{$(m_{2,j},m_{4,j})\gets(\tr A_j,\|A_j\|_F^2)/r$}}
\State \textcolor{method}{\mbox{$(m_{6,j},m_{8,j})\gets(\langle A_j,B_j\rangle_F,\|B_j\|_F^2)/r$}}
\State $C_j\gets b_jA_j+c_jB_j$
\State $X_{j+1}\gets a_jX_j+C_jX_j$ \hfill\textcolor{referenceblue}{(3)}
\EndFor
\State \textcolor{method}{Read $\|X_{K_{\mathrm{meas}}}\|_F^2/r$}
\end{algorithmic}

\textcolor{referenceblue}{(1)--(3) are existing matrix products.}\par
\textcolor{method}{Purple marks added scalar reductions.}\zsavepos{nsboxend}
\end{minipage}}
\end{wrapfigure}

Newton--Schulz already forms the matrices needed to read spectral moments. At iteration $j$, the intermediate matrices
$A_j=X_jX_j^\top$ and $B_j=A_j^2$
expose the even moments of orders two through eight using only scalar reductions. No additional matrix multiplication is required. Because each iteration applies a different polynomial transform, successive readouts measure different views of the same input spectrum \citep{weisse2006kpm,lin2016density}. To see this, write
$X_0=U\diag(\gamma_i)V^\top$
and define $z_0(x)=x$,
$z_{j+1}(x)=p_j(z_j(x))$, and
$Z_j=\diag(z_j(\gamma_i))$.
Then \cref{eq:loop} gives\zsavepos{nstextend}
\clearrefbox\WFclear
\begin{equation}
X_j=UZ_jV^\top,\quad
A_j=UZ_j^2U^\top,\quad
B_j=UZ_j^4U^\top.
\label{eq:gram-spectrum}
\end{equation}

Since $A_j$ is symmetric,
$\langle A_j,B_j\rangle_F=\tr(A_j^3)$ and
$\|B_j\|_F^2=\tr(A_j^4)$, giving
\begin{equation}
m_{2,j}=\tr(A_j)/r, \  
m_{4,j}=\|A_j\|_F^2/r, \ 
m_{6,j}=\langle A_j,B_j\rangle_F/r, \ 
m_{8,j}=\|B_j\|_F^2/r.
\label{eq:reductions}
\end{equation}
The cost of these reductions grows only with the smaller matrix dimension and is a vanishing fraction of each iteration (\cref{app:pretraining-cost}).

\begin{proposition}[Spectral moments from the iteration]
\label{prop:moments}
In exact arithmetic,
\begin{equation}
m_{2k,j}
=
\frac{\tr A_j^k}{r}
=
\frac1r\sum_{i=1}^r z_j(\gamma_i)^{2k}
=
\int z_j(x)^{2k}\dd\mu(x),
\qquad k=1,\ldots,4.
\label{eq:moments}
\end{equation}
For $j=0,\ldots,K_{\mathrm{meas}}-1$, this gives four measurements per iteration. The final output contributes
\[
m_{2,K_{\mathrm{meas}}}
=
\|X_{K_{\mathrm{meas}}}\|_F^2/r
=
\int z_{K_{\mathrm{meas}}}(x)^2\dd\mu(x),
\]
for $4K_{\mathrm{meas}}+1$ scalar measurements in total.
\end{proposition}

Each pair $(j,k)$ therefore defines a known measurement function
$\phi_q(x)=z_j(x)^{2k}$
with observation
$y_q=\int\phi_q\dd\mu$.
These functions constrain several transformed views of the same spectrum from one pass.
We choose $s$ using the CANS Gram-power upper bound
\citep[Section~3.3]{cans2025}, reusing the first-iteration products, which safely bounds the spectrum from above.

\subsection{Fitting a spectrum to the measured moments}

We recover a discrete spectral measure from the measurement functions using a
maximum-entropy moment fit \citep{mead1984,silver1997entropy}. Place masses
$w_b$ on grid points $g_b\in[0,1]$, and define $\Phi_{qb}=\phi_q(g_b)$. Then $\Phi w$ gives the predicted measurements. We enforce normalization and the known second moment by restricting $w$ to
\begin{equation}
\Delta=
\left\{
w\ge0:
\boldsymbol 1^\top w=1,\qquad
\sum_b w_bg_b^2=\frac{1}{rs^2}
\right\}.
\label{eq:simplex}
\end{equation}

Because discretization and measurement error may preclude an exact match, we
first minimize the weighted residual, then select the maximum-entropy fit
within tolerance $\kappa$. With measured values $\yh_q$, row scales $d_q$ that put them on a common scale, and $D=\diag(1/d_q)$,
\begin{equation}
\xi^\star=\min_{w\in\Delta}\|D(\Phi w-\yh)\|_2,
\qquad
\wh=\argmax_{w\in\Delta}\Bigl\{-\sum_b w_b\log w_b:\|D(\Phi w-\yh)\|_2\le\xi^\star+\kappa\Bigr\}.
\label{eq:fit}
\end{equation}

Both problems are convex. When several grid
distributions match the measurements, the entropy objective selects a unique representative.
\Cref{fig:observer} compares the fitted measures with the SVD spectra. The fit
captures broad attention and concentrated MLP spectra, with close
agreement in cumulative mass.

\begin{figure}[!htbp]
\centering
\includegraphics[width=\textwidth]{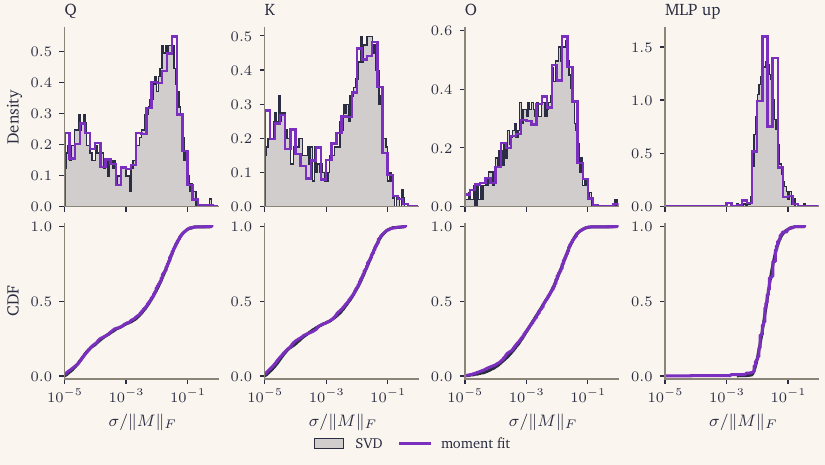}
\caption{\textbf{Moment fits recover useful spectral structure.}
Four matrices at layer $5$, step $16$k.
The top row shows the density in $\log_{10}(\sigma/\|M\|_F)$ and the bottom row
the CDF, both conditioned on $\sigma/\|M\|_F\ge10^{-5}$.}
\label{fig:observer}
\end{figure}
\Needspace{5\baselineskip}
\subsection{What the fitted spectrum can predict}

We do not need to recover the spectrum uniquely, only to predict the
spectral quantities used to score candidate routines. Although finitely many
moments may not identify a distribution \citep{krein1977}, they can determine
a particular spectral average when its integrand is well represented by the
measurements \citep{braverman2022sublinear}.

\begin{samepage}
\begin{theorem}[Prediction from the fitted spectrum]
\label{thm:transfer}

Let $\mu$ be the true spectral measure on $[0,1]$, and let
$\phi_1,\ldots,\phi_L$ be the measurement functions with observations
$\yh_q$ satisfying
$|\yh_q-\int \phi_q\,\dd\mu|\le\omega_q$.
Let $\wh$ solve \cref{eq:fit} and define
$\widehat\mu=\sum_b\wh_b\delta_{g_b}$.
For any bounded measurable $f$ and any affine combination of the measurement functions $\psi_\alpha=\alpha_0+\sum_{q=1}^L\alpha_q\phi_q$, define
$\beta_\alpha(f)=\sup_{x\in[0,1]}|f(x)-\psi_\alpha(x)|$.
Then, with $\alpha=(\alpha_1,\ldots,\alpha_L)$ and $D$, $\xi^\star$ as in \cref{eq:fit},
\begin{equation}
\left|
\int f\dd\widehat\mu-\int f\dd\mu
\right|
\le
2\beta_\alpha(f)
+\sum_{q=1}^L |\alpha_q|\omega_q
+(\xi^\star+\kappa)\|D^{-1}\alpha\|_2.
\label{eq:fit-transfer}
\end{equation}

\end{theorem}
\end{samepage}

The three terms capture how well $f$ is represented by the measured
functions, error in the observed values, and the fit residual. In
exact arithmetic the reductions are exact, so $\omega_q=0$ and the middle term
covers only numerical error. If $f$ lies in their affine span and the
fitted moments are exact, its spectral average is recovered exactly even when
the spectrum is not uniquely determined.

\subsection{Selecting a Polar Express routine}
\label{sec:method}

The fitted spectrum thus predicts the averages that determine candidate
quality, and we use it to score a precomputed set of PE routines. Each entry
$\pi=(\rho_\pi,K_\pi)\in\mathcal C$ is a complete PE sequence with response $P_\pi$,
designed for $[\rho_\pi,1]$. Here $\rho_\pi$ only sets the routine's design interval.
All candidates are scored on the same spectral interval
$\region=[x_\star/s,1]$, where $x_\star$ is fixed in Frobenius-normalized
coordinates. Their true and predicted squared RMS errors are
\begin{equation}
E_\region(P_\pi)^2
=
\frac{\int_\region (1-P_\pi(x))^2\dd\mu(x)}{\mu(\region)},
\qquad
\Eh_\pi^2
=
\frac{\sum_{g_b\in\region}\wh_b(1-P_\pi(g_b))^2}
     {\sum_{g_b\in\region}\wh_b}.
\label{eq:score}
\end{equation}
Since candidate responses are precomputed, scoring needs only scalar
operations on the fitted spectrum, with no candidate matrix runs. The same scores support two objectives. The \emph{fixed-depth
rule} picks the best routine at a given depth $K$,
$\widehat\pi_K\in\argmin_{\pi\in\mathcal C,\,K_\pi=K}\Eh_\pi$. The
\emph{target rule} picks the shortest routine predicted to meet an accuracy
$\tau$, $\widehat\pi_\tau\in\argmin_{\pi\in\mathcal C,\,\Eh_\pi\le\tau}K_\pi$,
and falls back to fixed PE at the largest allowed depth if no candidate meets
$\tau$. \Cref{fig:choice} illustrates both choices. 

These three steps, measure, fit, and select, make up spectral estimation.
Beyond early training, spectra change gradually, so we measure every $T$ steps and reuse the selected routine in between. Offline, we instead fit each saved checkpoint independently.

\begin{figure}[!htbp]
\centering
\includegraphics[width=\textwidth]{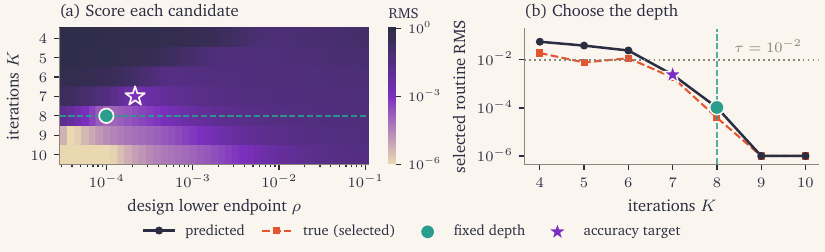}
\caption{\textbf{One set of predicted scores supports both objectives.}
MLP-up, layer $5$, step $16$k.
\textbf{(a)} Predicted errors, with the best depth-$8$ endpoint (circle) and the
shortest routine meeting $\tau=10^{-2}$ (star).
\textbf{(b)} Predicted and SVD-evaluated errors of the selected routines.
Values below $10^{-6}$ are shown at that floor.}
\label{fig:choice}
\end{figure}

\section{Experiments}
\label{sec:experiments}

\subsection{Orthogonalization of real spectra}

We evaluate saved Muon momentum spectra from a FineWeb run
\citep{penedo2024fineweb}. Baselines are fixed conservative PE on
$[10^{-5},1]$ \citep{amsel2026polar} and a minimax oracle. Using the exact
spectrum, the oracle selects from the same candidate set the routine with the
smallest worst-case error
$E_{\max}(P)=\max_{x\in[\gamma_{\min},\gamma_{\max}]}|1-P(x)|$, where
$[\gamma_{\min},\gamma_{\max}]$ spans the scored singular values. Spectral estimation instead minimizes
predicted RMS error, which the fitted spectrum predicts more reliably than
worst-case error, so we report both errors. Both share
normalization, candidate set, and scoring interval. Offline
moments are computed exactly from the saved spectra, and bf16 measurements give
nearly the same selections (\cref{app:pretraining-validation}).

At fixed depth, selection reduces RMS error most strongly on MLP
matrices (\cref{fig:experiments,tab:oracle-depth}). Because it targets RMS
error, spectral estimation has lower RMS error than the minimax oracle on attention
through $K=9$ and on MLP through $K=6$, while the oracle has lower worst-case
error at every depth.

With a target accuracy, spectral estimation picks shorter routines than fixed PE, especially for MLP matrices (\cref{fig:experiments}c--d, \cref{tab:allocation}).

\begin{figure}[H]
\centering
\includegraphics[width=\textwidth]{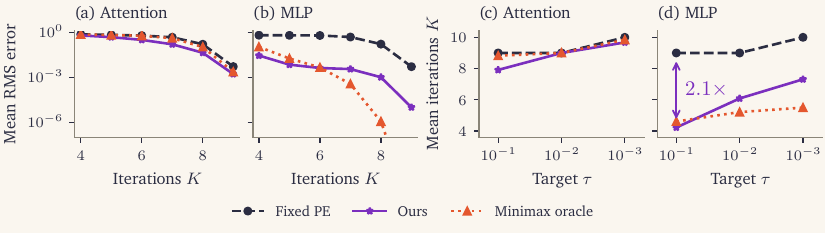}
\caption{\textbf{Offline accuracy and depth.} Means over attention and MLP
matrices, scored above $10^{-5}$ in Frobenius-normalized coordinates. \textbf{(a--b)} Polar RMS error at fixed
depth (\cref{tab:oracle-depth}).
\textbf{(c--d)} Iterations needed to meet a target $\tau$
(\cref{tab:allocation}). Ours meets $\tau$ in predicted RMS error, the
oracle in worst-case error, and fixed PE is the uniform depth that meets
$\tau$.}
\label{fig:experiments}
\end{figure}

\begin{table}[!htbp]
\centering\small
\caption{\textbf{Fixed-depth polar error.} Mean per-matrix RMS and worst-case error (minimax oracle).}
\label{tab:oracle-depth}
\setlength{\tabcolsep}{4pt}
\begin{tabular}{lcrrr@{\hspace{10pt}}rrr}
\toprule
& & \multicolumn{3}{c}{RMS error} & \multicolumn{3}{c}{Worst-case error}\\
\cmidrule(lr){3-5}\cmidrule(lr){6-8}
& $K$ & Fixed PE & Ours & Oracle & Fixed PE & Ours & Oracle\\
\midrule
Attention & $5$ & $0.667$ & $\best{0.472}$ & $0.653$ & $0.976$ & $0.982$ & $\best{0.966}$\\
 & $8$ & $0.159$ & $\best{0.0431}$ & $0.104$ & $0.222$ & $0.353$ & $\best{0.161}$\\
MLP & $5$ & $0.626$ & $\best{0.00694}$ & $0.0179$ & $0.976$ & $0.0643$ & $\best{0.0266}$\\
 & $8$ & $0.162$ & $0.00101$ & $\best{1.13\!\times\!10^{-6}}$ & $0.222$ & $0.00153$ & $\best{1.78\!\times\!10^{-6}}$\\
\bottomrule
\end{tabular}

\vspace{8pt}
\caption{\textbf{Mean depth $K$ at target $\tau$} (ours in predicted RMS error, oracle in worst-case error).}
\label{tab:allocation}
\setlength{\tabcolsep}{5pt}
\begin{tabular}{lrrr@{\hspace{12pt}}rrr}
\toprule
& \multicolumn{3}{c}{Attention} & \multicolumn{3}{c}{MLP}\\
\cmidrule(lr){2-4}\cmidrule(lr){5-7}
$\tau$ & $10^{-1}$ & $10^{-2}$ & $10^{-3}$ & $10^{-1}$ & $10^{-2}$ & $10^{-3}$\\
\midrule
Ours & $\best{7.91}$ & $\best{9.00}$ & $\best{9.68}$ & $\best{4.23}$ & $6.09$ & $7.31$\\
Oracle & $8.82$ & $\best{9.00}$ & $9.84$ & $4.61$ & $\best{5.21}$ & $\best{5.50}$\\
\bottomrule
\end{tabular}
\end{table}

\subsection{LLM pretraining}
\label{sec:training}

We pretrain GPT-style decoder-only language models \citep{radford2019language}
with $160$M, $300$M, and $1$B non-embedding parameters on $20$B tokens of
Nemotron-CC v2 High-Quality \citep{nvidia2025nemotronccv2,nvidia2025nemotronnano2}, using the
experimental setup of \citet{okls2026} with minor changes
(\cref{app:pretraining-setup}).
Muon and OKLS need different matrix functions. Muon needs the polar factor of
each momentum matrix, and OKLS the inverse square roots of Kronecker-factored
second-moment statistics, computed by a coupled Newton--Schulz iteration.
Our method handles both with the same measure, fit, and select procedure,
changing only the measured products and candidate routines
(\cref{app:inverse}), so it is not specific to
orthogonalization. Each run fixes the depth $K$ of ordinary
(non-measurement) steps and varies only the routine. With the fixed-depth
RMS rule of \cref{sec:method}, our method selects among $K$-iteration routines
that differ only in their endpoint $\rho$, namely $\PE[\rho,1]$ with
$\rho\in[10^{-5},10^{-1}]$ for Muon and cubic inverse-root routines with
$\rho\in[10^{-7},1]$ for OKLS. The oracle selects from the same set by
the minimax rule every $500$ steps. The fixed baseline is $\PE[10^{-5},1]$ for Muon and, for OKLS, the cubic routine for eigenvalues above $10^{-7}\lambda_{\max}$.

\begin{table}[!htbp]
\centering\small
\begin{minipage}[b]{0.66\textwidth}
\caption{\textbf{Pretraining across optimizers and sizes.}
Final validation loss, $K=5$, $20$B tokens, with $T=200$ for ours.}
\label{tab:pretraining-main}
\end{minipage}\hfill
\begin{minipage}[b]{0.30\textwidth}
\caption{\textbf{Muon 160M across depths.} Final validation loss.}
\label{tab:depth}
\end{minipage}\par
\begin{minipage}[t]{0.66\textwidth}
\centering
\setlength{\tabcolsep}{3pt}
\begin{tabular}[t]{lrrrcrrr}
\toprule
& \multicolumn{3}{c}{Muon} && \multicolumn{3}{c}{OKLS}\\
\cmidrule(lr){2-4}\cmidrule(lr){6-8}
Size & Fixed & Oracle & Ours && Fixed & Oracle & Ours\\
\midrule
160M & $2.761$ & $2.752$ & $\best{2.751}$ && $2.687$ & $2.675$ & $\best{2.673}$\\
300M & $2.662$ & $\best{2.654}$ & $2.655$ && $2.597$ & $2.591$ & $\best{2.590}$\\
1B & $2.455$ & $2.450$ & $\best{2.449}$ && $2.449$ & $2.439$ & $\best{2.436}$\\
\bottomrule
\end{tabular}
\end{minipage}\hfill
\begin{minipage}[t]{0.30\textwidth}
\centering
\setlength{\tabcolsep}{4pt}
\begin{tabular}[t]{crr}
\toprule
$K$ & Fixed & Ours\\
\midrule
$4$ & $2.808$ & $\best{2.791}$\\
$5$ & $2.761$ & $\best{2.751}$\\
$6$ & $2.754$ & $\best{2.752}$\\
$8$ & $\best{2.752}$ & $\best{2.752}$\\
\bottomrule
\end{tabular}
\end{minipage}
\end{table}

Spectral estimation improves on the fixed baseline for both optimizers at
every model size, and it matches or improves on periodic selection from an
exact SVD or eigendecomposition (\cref{tab:pretraining-main,fig:pretraining}).

We also vary the depth $K$ of ordinary Muon steps at $160$M
(\cref{tab:depth}). The fixed baseline saturates quickly. With $6$ and $8$ iterations it
reaches the same validation loss, so more accurate orthogonalization brings no
further gain. With only five iterations, our method reaches a lower loss than the fixed
baseline at any depth, so adaptation can replace extra iterations.

\begin{figure}[t]
\centering
\includegraphics[width=\textwidth]{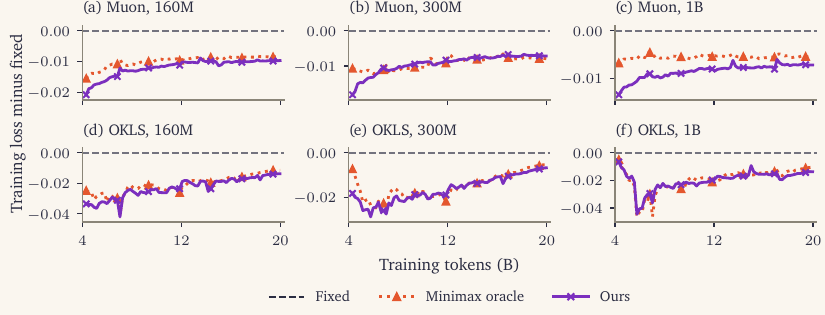}
\caption{\textbf{Training loss relative to the fixed baseline.}
Training loss minus the fixed baseline at $K=5$ in $50$-step windows, where negative is better.}
\label{fig:pretraining}
\end{figure}

\section{Conclusion}

The Newton--Schulz method does not need to treat the spectrum it acts on as unknown, because its iterations already produce measurements that can be reused to specialize its
own computations. We use these measurements to select spectrum-adapted polynomial routines, improving approximation at fixed depth or reducing the iterations needed to reach a target accuracy. We evaluate the method on saved momentum spectra and in end-to-end GPT pretraining. Useful adaptation does not require recovering the spectrum exactly, only
predicting which candidate routine will work well. Because the measurements come from matrices the iteration already forms and the fit works on a fixed number of scalars, this adaptation stays cheap at the matrix sizes of current frontier models, where computing exact spectra becomes expensive.

\FloatBarrier
{\small\bibliography{references}\bibliographystyle{iclr2027_conference}}
\clearpage
\phantomsection\label{paper:appendix-start}
\appendix
\input{appendix}
\end{document}

%% file: appendix.tex
\renewcommand{\topfraction}{0.9}
\renewcommand{\bottomfraction}{0.8}
\renewcommand{\textfraction}{0.07}
\renewcommand{\floatpagefraction}{0.85}
\setcounter{topnumber}{3}
\setcounter{bottomnumber}{2}
\setcounter{totalnumber}{5}

\section{Proofs}

\subsection{\texorpdfstring{Proof of \cref{prop:moments}}{Proof of Proposition 1}}
\label{app:metric}

For a thin SVD $X=U\diag(x_i)V^\top$ with orthonormal columns in $U$ and $V$,
$(XX^\top)^kX=U\diag(x_i^{2k+1})V^\top$. Applying this identity iteration by iteration
gives \cref{eq:scalar}.

\begin{proof}
The eigenvalues of $A_j$ are $z_j(\gamma_i)^2$, so
$\tr A_j^k=\sum_i z_j(\gamma_i)^{2k}$. For symmetric $A_j$,
\[
 \tr A_j^2=\|A_j\|_F^2,\qquad
 \tr A_j^3=\langle A_j,A_j^2\rangle_F,\qquad
 \tr A_j^4=\|A_j^2\|_F^2.
\]
The final iterate satisfies
$\|X_{K_{\mathrm{meas}}}\|_F^2=\sum_i z_{K_{\mathrm{meas}}}(\gamma_i)^2$.
This proves \cref{eq:moments}.
\end{proof}

\paragraph{Singular values below the scoring cutoff.}
Let $x_\star$ be the scoring cutoff, $\region=[x_\star/s,1]$ the scoring
interval, and $m_\region=\mu(\region)$ the fraction of singular values in it.
For a full-row-rank input, the output is $X_K=U\diag(P(\gamma_i))V^\top$ and
the polar factor is $UV^\top$, so
\begin{equation}
 \frac{\|X_K-UV^\top\|_F^2}{r}
 =m_\region E_\region(P)^2
 +\frac1r\sum_{\gamma_i\notin\region}(1-P(\gamma_i))^2.
 \label{eq:full-error}
\end{equation}
Our score covers only the first term. In the offline data, on average
$3.8\%$ of attention singular values lie below the cutoff (median $2.0\%$, at
most $24\%$), and no MLP singular values do. For rank-deficient inputs the same
decomposition holds with the partial polar factor, whose response is zero on
zero singular values.

\subsection{\texorpdfstring{Proof of \cref{thm:interval-cost}}{Proof of Theorem 1}}
\label{app:interval}

\begin{proof}
We first derive the recursion using the composition-optimality result of
\citet[Theorem~3.1]{amsel2026polar}. The change of variables $x=ut$ maps a
quintic on $[\ell,u]$ to one on $[\ell/u,1]$, so the one-iteration optimum depends
only on $\rho=\ell/u$. If an optimal iteration has error $\delta$, its image is
$[1-\delta,1+\delta]$ by continuity and equioscillation. The next optimal error
is therefore $e((1-\delta)/(1+\delta))$, which gives
\cref{eq:interval-recursion}.

For strict monotonicity, $\operatorname{span}\{x,x^3,x^5\}$ is a Haar space on
every positive interval, so the minimax residual has four alternating extrema.
Restricting a minimizer from $[\rho_1,1]$ to $[\rho_2,1]$ with $\rho_1<\rho_2$
removes the leftmost extremum, so it cannot remain minimax with the same error.
Hence the one-iteration error is strictly decreasing in $\rho$. Since
$t\mapsto(1-t)/(1+t)$ is strictly decreasing, induction on $K$ gives strict
monotonicity at every depth.

For the lower bound, let $P$ be an odd polynomial of degree at most $d$ and
$E=\max_{x\in[\rho,1]}|1-P(x)|$. The case $E\ge1$ is immediate. Otherwise
$P(x)\in[1-E,1+E]$, and
$R(y)=1-P(\sqrt y)^2/(1+E^2)$ is a polynomial of degree at most $d$ with
$R(0)=1$ and $\max_{y\in[\rho^2,1]}|R(y)|\le 2E/(1+E^2)$. Mapping $[\rho^2,1]$
to $[-1,1]$ and applying the Chebyshev extremal bound at the image of $y=0$
gives
\[
 \frac{2E}{1+E^2}\ge \frac{1}{\cosh(dL)},
 \qquad
 L=\log\frac{1+\rho}{1-\rho}.
\]
Because $e\mapsto2e/(1+e^2)$ is increasing on $[0,1]$ and equals
$1/\cosh(dL)$ at $e=\exp(-dL)$, we get
$E\ge\exp(-dL)=\left((1-\rho)/(1+\rho)\right)^d$. A composition of $K$
quintics has degree at most $5^K$, which proves \cref{eq:interval-lower}.
\end{proof}

Consequently, reaching a uniform error $\tau$ on $[\rho,1]$ requires
\begin{equation}
 5^K\ge\frac{\log(1/\tau)}{\log((1+\rho)/(1-\rho))}.
 \label{eq:interval-necessary-depth}
\end{equation}

\subsection{\texorpdfstring{Proof of \cref{thm:transfer} and accuracy of selection}{Proof of Theorem 2 and accuracy of selection}}
\label{app:transfer}

\begin{proof}
The set $\Delta$ is a compact convex subset of the simplex, so the best-fit
minimum in \cref{eq:fit} is attained, and its minimizer is feasible
for the maximum-entropy problem. The feasible set of the latter is compact and convex,
and entropy is continuous and strictly concave, so the fitted measure exists
and is unique.

Insert $\psi_\alpha=\alpha_0+\sum_q\alpha_q\phi_q$ between the two integrals.
The constant $\alpha_0$ cancels because both measures have unit mass, so
\begin{align*}
 \int f\dd\widehat\mu-\int f\dd\mu
 &=\int(f-\psi_\alpha)\dd\widehat\mu-\int(f-\psi_\alpha)\dd\mu\\
 &\quad+\alpha^\top(\Phi\wh-\yh)
 +\alpha^\top\!\left(\yh-\int\phi\dd\mu\right).
\end{align*}
The first two terms are at most $2\beta_\alpha(f)$ in absolute value, and the
last is at most $\sum_q|\alpha_q|\omega_q$. Finally,
$|\alpha^\top(\Phi\wh-\yh)|
 =|(D^{-1}\alpha)^\top D(\Phi\wh-\yh)|
 \le\|D^{-1}\alpha\|_2(\xi^\star+\kappa)$.
Adding the terms gives \cref{eq:fit-transfer}.
\end{proof}

\paragraph{From integrals to selection.}
The score in \cref{eq:score} is the ratio of the error integral
$J_\pi=\int_\region(1-P_\pi)^2\dd\mu$ to the scored mass $m_\region$.
\Cref{thm:transfer} applies to each. If the estimates satisfy
$|\widehat J_\pi-J_\pi|\le\delta_{J_\pi}$, $|\widehat m_\region-m_\region|\le\delta_m$,
and $\widehat m_\region\ge m_0>0$, and $(1-P_\pi)^2\le C_\pi$ on $\region$, then
\begin{equation}
 \left|\frac{\widehat J_\pi}{\widehat m_\region}-\frac{J_\pi}{m_\region}\right|
 \le\frac{\delta_{J_\pi}+C_\pi\delta_m}{m_0}.
 \label{eq:ratio-perturbation}
\end{equation}
Accurate scores give accurate selection. Write $E_\pi=E_\region(P_\pi)$, let
$\widehat\pi$ minimize the predicted squared error $\Eh_\pi^2$ and $\pi^\star$
the true one over the same candidates. If $|\Eh_\pi^2-E_\pi^2|\le\delta_\pi$ for
every candidate, then
\begin{equation}
 E_{\widehat\pi}^2-E_{\pi^\star}^2\le\delta_{\widehat\pi}+\delta_{\pi^\star},
 \label{eq:selection-regret}
\end{equation}
since
$E_{\widehat\pi}^2\le\Eh_{\widehat\pi}^2+\delta_{\widehat\pi}
\le\Eh_{\pi^\star}^2+\delta_{\widehat\pi}
\le E_{\pi^\star}^2+\delta_{\pi^\star}+\delta_{\widehat\pi}$.

\FloatBarrier
\section{Method details}

\subsection{Input normalization}
\label{app:scaling}

For $\bar X=M/\|M\|_F$ and $\bar A=\bar X\bar X^\top$, the Gram-power bound of
\citet[Section~3.3]{cans2025} is
\begin{equation}
 s_8=\|\bar A^2\|_F^{1/4}=\Bigl(\sum_i x_i^8\Bigr)^{1/8},
 \qquad
 \max_i x_i\le s_8\le r^{1/8}\max_i x_i,
 \label{eq:scale}
\end{equation}
and $s_8\le1$ because $\sum_i x_i^2=1$. Rescaling does not change the polar
factor. The first-iteration products give both the scale and the moments. If $X_0=\bar X/s$, then $A_0=\bar A/s^2$ and $A_0^2=\bar A^2/s^4$, so each iteration-zero moment is a rescaled reduction of $\bar A$ or $\bar A^2$. The same
scale is used for the measurement pass, the spectral grid, and scoring.
\Cref{fig:scaling} shows that this bound is nearly tight for every matrix
type, unlike the Frobenius norm, and that better normalization alone accounts
for only part of the gain from spectral selection.

\begin{figure}[!htbp]
\centering
\includegraphics[width=\textwidth]{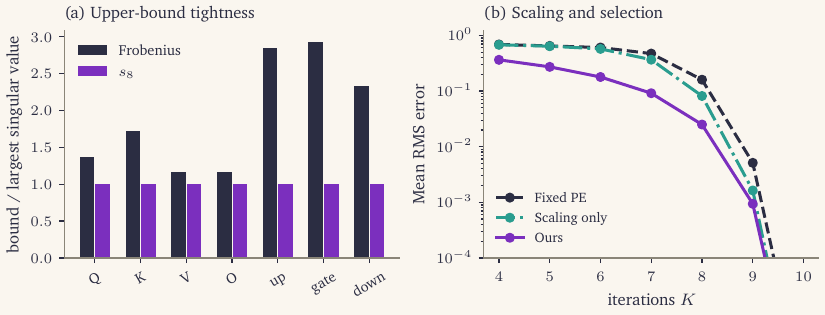}
\caption{\textbf{Normalization explains part of the gain.} \textbf{(a)} Ratio of
each upper bound to the largest singular value, by matrix type. \textbf{(b)}
Fixed PE, fixed PE with the Gram-power scale (scaling only), and spectral estimation.}
\label{fig:scaling}
\end{figure}

\subsection{Candidates and fitting}
\label{app:dictionary}

Each candidate is a complete PE routine given by a design lower endpoint and a
depth. Candidates are checked for positive iteration responses. For an iteration $p(x)=x(a+bx^2+cx^4)$ on $[0,u]$, positivity reduces to the
quadratic $a+bt+ct^2$ on $[0,u^2]$, and the extreme outputs occur at the
endpoints or at roots of $a+3bx^2+5cx^4$. Candidates that fail are removed.

The convex programs in \cref{eq:fit} use the row scaling
$D=\diag(1/d_q)$ with $d_q=\max\{\max_b|\Phi_{qb}|,|\yh_q|,d_{\min}\}$, which
puts measurements of different magnitudes on a common scale. Each matrix is
fitted independently. \Cref{tab:polar-settings} lists all settings.

\paragraph{Ambiguity of the fitted spectrum.}
Finitely many measurements can admit many spectral measures.
\Cref{fig:nonid} constructs two grid measures that match the same
measurements to numerical tolerance, the maximum-entropy fit and an extreme
point of the feasible set. Their point masses differ substantially, but their
cumulative distributions nearly coincide, so spectral averages such as the RMS
score are similar under both (\cref{thm:transfer}). Maximum entropy selects one
representative among the feasible measures.

\begin{figure}[!htbp]
\centering
\includegraphics[width=\textwidth]{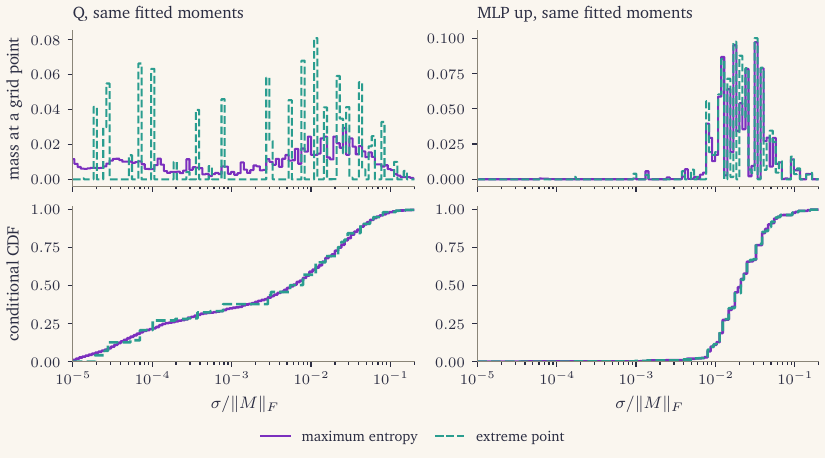}
\caption{\textbf{Different spectral measures can match the same measurements.}
Maximum-entropy and extreme-point grid measures for representative attention
and MLP spectra. The top row shows the mass on the estimator grid and the bottom
row the conditional CDF on the scoring interval.}
\label{fig:nonid}
\end{figure}

\begin{table}[!htbp]
\centering\small
\caption{\textbf{Estimator and candidate settings.}}
\label{tab:polar-settings}
\begin{tabular}{ll}
\toprule
Component & Setting\\
\midrule
Scale & $1.005\,s_8$, rounded up to the grid $10^{-j/24}$, $j=0,\ldots,15$\\
Scoring cutoff & $x_\star=10^{-5}$ in Frobenius-normalized coordinates\\
Spectral grid & zero, $24$ points below the cutoff, $120$ in the scoring interval\\
Measurements & four moments per iteration ($0$--$9$) and the final second moment\\
Entropy tolerance & $\kappa=3\times10^{-4}$\\
Row-scale floor & $d_{\min}=10^{-10}$\\
Solver & CVXPY with Clarabel\\
Design lower endpoints & $49$ log-spaced values in $[10^{-5},10^{-1}]$\\
Candidate depths & $K=4,\ldots,10$\\
Remez tolerance / iteration cap & $10^{-15}$ / $300$\\
Endpoint cushion & $\ell\leftarrow\max(\ell,0.02u)$\\
Coefficient safety scale & $(1.01,1.01^3,1.01^5)$ before the final iteration\\
Positivity margin & $10^{-9}$\\
\bottomrule
\end{tabular}
\end{table}

\subsection{Periodic measurement}
\label{app:schedule}

Periodic recomputation in the background is a standard way to amortize
preconditioner updates \citep{anil2020scalable}. \Cref{alg:schedule} applies it here. A measurement pass produces the current update and the measurements, and
fitting and selection run in the background while the previous routine stays
in use.

\begin{algorithm}[!htbp]
\caption{Periodic spectral estimation.}
\label{alg:schedule}
\small
\makeatletter
\def\theHALG@line{schedule.\arabic{ALG@line}}
\makeatother
\begin{algorithmic}[1]
\Require Period $T$, measurement routine $\mathcal S_{\mathrm{meas}}$ of depth $K_{\mathrm{meas}}$, candidate set, selection rule
\State Use $\mathcal S_{\mathrm{meas}}$ for every matrix
\For{optimizer steps $t=0,1,\ldots$}
  \State Form the current momenta and their normalization
  \If{a finished fit is available}
    \State Switch to the newly selected routines
  \EndIf
  \If{a measurement is due and no fit is running}
    \State Run $\mathcal S_{\mathrm{meas}}$, record the moments, and use its output for the update
    \State Start fitting and selection in the background
  \Else
    \State Apply the current routines
  \EndIf
\EndFor
\end{algorithmic}
\end{algorithm}

If steps between measurements use mean depth $\bar K$, the average number of
iterations per step is
\begin{equation}
 K_{\rm eff}=\bar K+\frac{K_{\mathrm{meas}}-\bar K}{T}.
 \label{eq:amortized}
\end{equation}

\FloatBarrier
\section{Offline experiments}
\label{app:data}

\subsection{Data}

We use Muon momentum matrices saved during a FineWeb run of a small GPT model
(\cref{tab:offline-data}). The held-out run (\cref{app:heldout}) and
\cref{fig:spectra} use the $160$M Muon baseline of our pretraining experiments
(\cref{app:training}).

\begin{table}[!htbp]
\centering\small
\caption{\textbf{Offline data.}}
\label{tab:offline-data}
\begin{tabular}{lll}
\toprule
& Primary run & Held-out run\\
\midrule
Dataset & FineWeb & Nemotron-CC v2 High-Quality\\
Model & $12$ layers, width $768$ & $24$ layers, width $768$ ($160$M Muon baseline)\\
Checkpoints & every $2{,}000$ steps, $2$k--$30$k & step $1$, then every $500$ steps to $4{,}500$\\
Matrices & $720$ attention, $540$ MLP & $1{,}680$\\
Matrix types & \multicolumn{2}{l}{Q, K, V, O, MLP up, MLP gate, MLP down}\\
Muon momentum & $0.95$ & $0.9586$\\
\bottomrule
\end{tabular}
\end{table}

\subsection{Fitted spectra across layers and training}
\label{app:ridge-spectra}

\Cref{fig:ridge-layers,fig:ridge-time} compare the fitted and SVD spectra for
every matrix type, across layers and across training. The fits follow the
SVD spectra closely throughout.

\begin{figure}[!htbp]
\centering
\includegraphics[width=\textwidth]{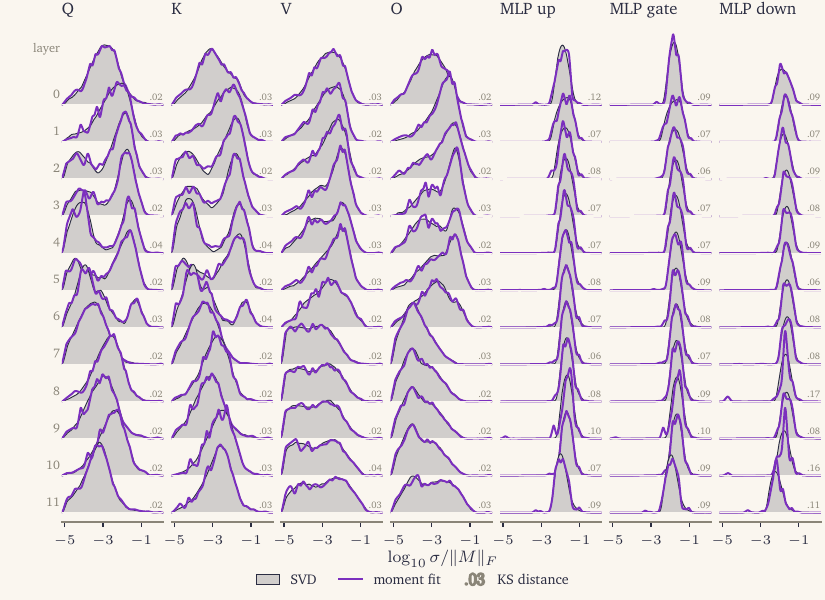}
\caption{\textbf{Fitted and SVD spectra across layers} (step $16$k), with the SVD in
grey and the moment fit in purple. Both are restricted to $\sigma/\|M\|_F\ge10^{-5}$ and
smoothed in $\log_{10}\sigma$. The numbers give the largest gap between the two
CDFs.}
\label{fig:ridge-layers}
\end{figure}

\begin{figure}[!htbp]
\centering
\includegraphics[width=\textwidth]{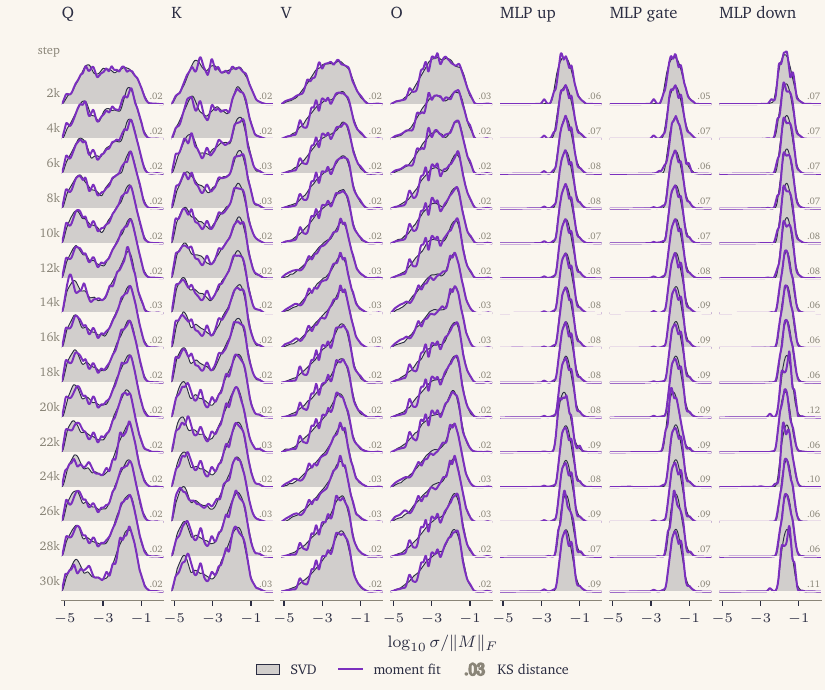}
\caption{\textbf{Fitted and SVD spectra across training} (layer $5$). Same
conventions as \cref{fig:ridge-layers}.}
\label{fig:ridge-time}
\end{figure}

\subsection{Fixed-depth errors and the choice of score}
\label{app:fixed}

\Cref{tab:fixed} extends \cref{tab:oracle-depth} to all depths and adds the RMS
oracle, which selects by exact RMS error. The minimax oracle, used as the
oracle throughout the paper, selects by exact worst-case error. On attention
matrices spectral estimation nearly matches the RMS oracle, and on MLP matrices
it is close at shallow depths. It has lower RMS error than the minimax oracle
on attention through $K=9$ and on MLP through $K=6$, and higher worst-case
error.

We score candidates by RMS error because the fitted spectrum can predict it.
RMS error is an average over the spectrum, which \cref{thm:transfer} covers.
Worst-case error depends on the extreme scored singular values, which finitely
many moments do not determine. \Cref{tab:minimax-fit} instead applies the
minimax rule to the fitted spectrum, on the interval that keeps all but $0.5/r$
of the fitted mass in each tail. Against its own oracle, it has much larger
regret than the RMS rule has against the RMS oracle, and it misses about a fifth
of the MLP accuracy targets, compared with at most $6\%$ for the RMS rule.

\begin{table}[!htbp]
\centering\small
\caption{\textbf{Scoring rules on the fitted spectrum.} Regret is the mean error of the selected routine minus that of the matching oracle, in the rule's own
criterion (worst-case error for minimax, RMS error for ours). Misses are the share of MLP matrices whose true error exceeds $\tau$ under the target rule.}
\label{tab:minimax-fit}
\setlength{\tabcolsep}{4pt}
\begin{tabular}{lrrrr@{\hspace{10pt}}rrr}
\toprule
& \multicolumn{4}{c}{Regret} & \multicolumn{3}{c}{MLP misses (\%)}\\
\cmidrule(lr){2-5}\cmidrule(lr){6-8}
& \multicolumn{2}{c}{Attention} & \multicolumn{2}{c}{MLP} & \multicolumn{3}{c}{$\tau$}\\
Rule on the fit & $K=5$ & $K=8$ & $K=5$ & $K=8$ & $10^{-1}$ & $10^{-2}$ & $10^{-3}$\\
\midrule
Minimax & $0.0014$ & $0.0038$ & $0.283$ & $0.0069$ & $20.0$ & $18.7$ & $21.1$\\
RMS (ours) & $\best{3\!\times\!10^{-5}}$ & $\best{0.00068}$ & $\best{0.0036}$ & $\best{0.0010}$ & $\best{0.0}$ & $\best{5.7}$ & $\best{5.4}$\\
\bottomrule
\end{tabular}
\end{table}

\begin{table}[!htbp]
\centering\footnotesize
\caption{\textbf{Fixed-depth RMS and worst-case error.} Means over attention
and MLP matrices. Oracle is the minimax oracle, and the RMS oracle selects by
exact RMS error. All methods use the same candidates and scoring interval.}
\label{tab:fixed}
\setlength{\tabcolsep}{2pt}
\begin{tabular}{crrrr@{\hspace{6pt}}rrr}
\toprule
& \multicolumn{4}{c}{RMS error} & \multicolumn{3}{c}{Worst-case error}\\
\cmidrule(lr){2-5}\cmidrule(lr){6-8}
$K$ & Fixed PE & Ours & Oracle & RMS oracle & Fixed PE & Ours & Oracle\\
\midrule
\multicolumn{8}{l}{\emph{Attention}}\\
$4$ & $0.729$ & $\best{0.618}$ & $0.718$ & $\best{0.618}$ & $0.994$ & $0.995$ & $\best{0.992}$\\
$5$ & $0.667$ & $\best{0.472}$ & $0.653$ & $\best{0.472}$ & $0.976$ & $0.982$ & $\best{0.966}$\\
$6$ & $0.604$ & $\best{0.311}$ & $0.583$ & $\best{0.311}$ & $0.904$ & $0.934$ & $\best{0.868}$\\
$7$ & $0.465$ & $0.158$ & $0.402$ & $\best{0.157}$ & $0.665$ & $0.774$ & $\best{0.582}$\\
$8$ & $0.159$ & $0.0431$ & $0.104$ & $\best{0.0424}$ & $0.222$ & $0.353$ & $\best{0.161}$\\
$9$ & $0.00515$ & $0.00164$ & $0.00229$ & $\best{0.00144}$ & $0.00753$ & $0.0121$ & $\best{0.00496}$\\
$10$ & $3.86\!\times\!10^{-7}$ & $1.55\!\times\!10^{-7}$ & $1.52\!\times\!10^{-7}$ & $\best{1.49\!\times\!10^{-7}}$ & $\best{5.56\!\times\!10^{-7}}$ & $1.15\!\times\!10^{-6}$ & $1.14\!\times\!10^{-6}$\\
\midrule
\multicolumn{8}{l}{\emph{MLP}}\\
$4$ & $0.638$ & $0.0281$ & $0.106$ & $\best{0.0252}$ & $0.994$ & $0.329$ & $\best{0.157}$\\
$5$ & $0.626$ & $0.00694$ & $0.0179$ & $\best{0.00330}$ & $0.976$ & $0.0643$ & $\best{0.0266}$\\
$6$ & $0.612$ & $0.00416$ & $0.00449$ & $\best{7.65\!\times\!10^{-4}}$ & $0.904$ & $0.0207$ & $\best{0.00647}$\\
$7$ & $0.482$ & $0.00352$ & $3.60\!\times\!10^{-4}$ & $\best{1.08\!\times\!10^{-4}}$ & $0.665$ & $0.00968$ & $\best{5.48\!\times\!10^{-4}}$\\
$8$ & $0.162$ & $0.00101$ & $1.13\!\times\!10^{-6}$ & $\best{7.03\!\times\!10^{-7}}$ & $0.222$ & $0.00153$ & $\best{1.78\!\times\!10^{-6}}$\\
$9$ & $0.00511$ & $1.00\!\times\!10^{-5}$ & $\best{2.06\!\times\!10^{-12}}$ & $\best{2.06\!\times\!10^{-12}}$ & $0.00753$ & $1.55\!\times\!10^{-5}$ & $\best{2.53\!\times\!10^{-12}}$\\
$10$ & $3.86\!\times\!10^{-7}$ & $3.99\!\times\!10^{-10}$ & $1.30\!\times\!10^{-16}$ & $\best{2.79\!\times\!10^{-18}}$ & $5.56\!\times\!10^{-7}$ & $5.16\!\times\!10^{-10}$ & $\best{2.26\!\times\!10^{-18}}$\\
\bottomrule
\end{tabular}
\end{table}

\subsection{Held-out run}
\label{app:heldout}

We apply the method unchanged to the held-out run (\cref{tab:heldout}). On
attention matrices spectral estimation stays close to the RMS oracle
and well below fixed PE, and on this run it also has lower worst-case error
than fixed PE. On MLP matrices it stays far below fixed PE, while the minimax
oracle is lower in both errors.

\begin{table}[!htbp]
\centering\footnotesize
\caption{\textbf{Held-out run.} Mean errors over all checkpoints with unchanged settings. Oracle is the minimax oracle.}
\label{tab:heldout}
\setlength{\tabcolsep}{2pt}
\begin{tabular}{crrrr@{\hspace{6pt}}rrr}
\toprule
& \multicolumn{4}{c}{RMS error} & \multicolumn{3}{c}{Worst-case error}\\
\cmidrule(lr){2-5}\cmidrule(lr){6-8}
$K$ & Fixed PE & Ours & Oracle & RMS oracle & Fixed PE & Ours & Oracle\\
\midrule
\multicolumn{8}{l}{\emph{Attention}}\\
$5$ & $0.640$ & $0.153$ & $0.348$ & $\best{0.149}$ & $0.976$ & $0.550$ & $\best{0.516}$\\
 $8$ & $0.161$ & $0.0123$ & $0.0288$ & $\best{0.00977}$ & $0.222$ & $0.126$ & $\best{0.0444}$\\
\midrule
\multicolumn{8}{l}{\emph{MLP}}\\
$5$ & $0.623$ & $0.0315$ & $0.0247$ & $\best{0.0214}$ & $0.976$ & $0.0539$ & $\best{0.0373}$\\
 $8$ & $0.162$ & $0.00232$ & $4.79\!\times\!10^{-6}$ & $\best{1.87\!\times\!10^{-6}}$ & $0.222$ & $0.00339$ & $\best{7.60\!\times\!10^{-6}}$\\
\bottomrule
\end{tabular}
\end{table}

\subsection{Ablations}
\label{app:ablations}

\paragraph{Measurements.}
\Cref{fig:moment-ablation} fits the spectrum from fewer measurements, either
from only the first iterations of the measurement pass or with fewer moment
orders per iteration. More iterations and more orders both help, mostly on MLP matrices.

\paragraph{Scoring cutoff.}
\Cref{fig:care-floor-ablation} repeats the offline comparison with lower
scoring cutoffs. A lower cutoff includes more of the small-singular-value tail
and makes the problem harder for every method.

\paragraph{Estimator settings.}
\Cref{tab:sensitivity} changes one setting at a time. Attention results are
insensitive to all of them. On MLP matrices at $K=5$, a small positive entropy
tolerance, $\kappa\in[10^{-4},3\times10^{-4}]$, gives the lowest error, and we use
$\kappa=3\times10^{-4}$ for all experiments.

\begin{figure}[!htbp]
\centering
\includegraphics[width=\textwidth]{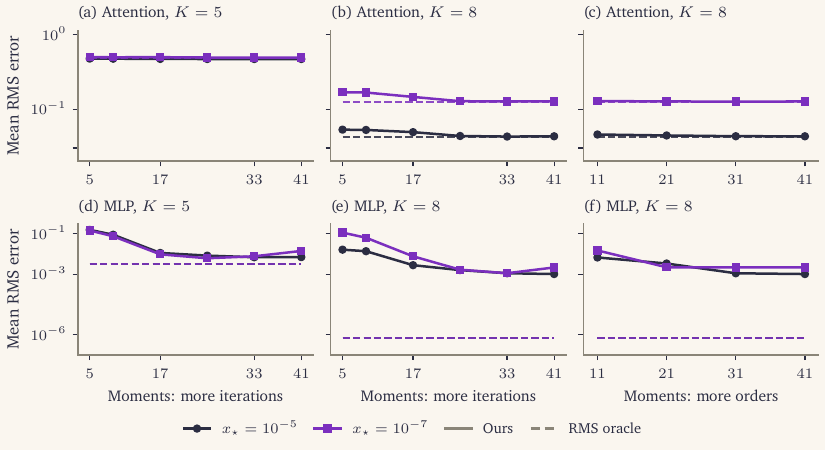}
\caption{\textbf{Fewer measurements.} The left and center columns use measurements from
only the first iterations of the measurement pass, and the right column fewer
moment orders at every iteration. Colors give $x_\star$, the scoring cutoff and
smallest design endpoint. Solid lines show spectral estimation and dashed lines
the RMS oracle.}
\label{fig:moment-ablation}
\end{figure}

\begin{figure}[!htbp]
\centering
\includegraphics[width=\textwidth]{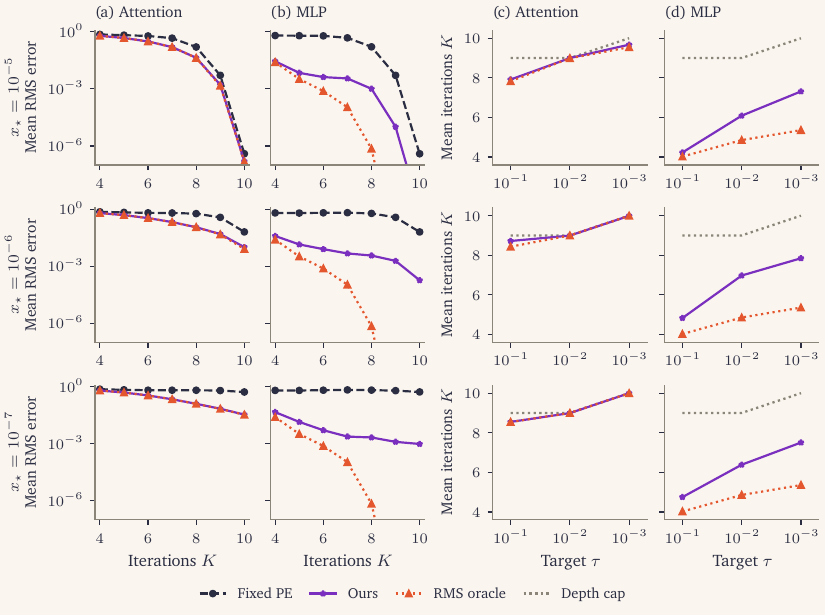}
\caption{\textbf{Lower scoring cutoffs.} Each row repeats \cref{fig:experiments}
with the RMS oracle at the cutoff $x_\star$ shown on the left. Panels \textbf{(a--b)} show RMS
error at fixed depth and \textbf{(c--d)} the mean depth at target $\tau$, capped
at the largest allowed depth (dotted). At the lower cutoffs, almost no attention matrix
meets $\tau\le10^{-2}$ within the cap, for either method.}
\label{fig:care-floor-ablation}
\end{figure}

\begin{table}[!htbp]
\centering\footnotesize
\caption{\textbf{Estimator settings.} Mean RMS error of the selected routine,
changing one setting at a time. RMS-oracle rows use the corresponding
candidate set.}
\label{tab:sensitivity}
\begin{tabular}{llrrrr}
\toprule
Setting & Value & Attn $K{=}5$ & Attn $K{=}8$ & MLP $K{=}5$ & MLP $K{=}8$ \\
\midrule
Default & \cref{tab:polar-settings} & 0.472 & 0.0431 & 0.00694 & 0.00101 \\
\midrule
Entropy tolerance $\kappa$ & $0$ (best fit only) & 0.472 & 0.0437 & 0.00729 & 0.000622 \\
 & $10^{-4}$ & 0.472 & 0.0430 & 0.00684 & 0.000777 \\
 & $10^{-3}$ & 0.472 & 0.0434 & 0.00862 & 0.00140 \\
 & $3\times10^{-3}$ & 0.472 & 0.0439 & 0.0131 & 0.00216 \\
Grid, scoring interval & 60 & 0.472 & 0.0435 & 0.0199 & 0.00223 \\
 & 240 & 0.472 & 0.0437 & 0.00870 & 0.00161 \\
Grid, below cutoff & 12 & 0.472 & 0.0434 & 0.00696 & 0.00112 \\
 & 48 & 0.472 & 0.0430 & 0.00692 & 0.000978 \\
Endpoints per decade & 6 & 0.473 & 0.0445 & 0.00672 & 0.00105 \\
 & 24 & 0.472 & 0.0429 & 0.00704 & 0.000986 \\
\midrule
RMS oracle & 6 per decade & 0.473 & 0.0437 & 0.00346 & $1.17\times10^{-6}$ \\
 & 12 per decade & 0.472 & 0.0424 & 0.00330 & $7.03\times10^{-7}$ \\
 & 24 per decade & 0.472 & 0.0422 & 0.00326 & $6.59\times10^{-7}$ \\
\bottomrule
\end{tabular}
\end{table}

\subsection{Target rule}
\label{app:target-reliability}

The target rule picks the shortest candidate whose predicted RMS error is at
most $\tau$. \Cref{tab:target-margin} adds a safety margin $m$ by requiring the
prediction to be at most $\tau/m$. On attention matrices the predictions are
accurate, and a small margin removes every miss at little extra depth. On MLP
matrices a small fraction of matrices miss the target even with a large margin.
These are MLP-down matrices with a few isolated singular values far below the
rest of the spectrum. Each such value adds a large error, and the moments do
not detect it.

\begin{table}[!htbp]
\centering\footnotesize
\caption{\textbf{Target rule with a safety margin $m$.} A candidate is accepted
if its predicted RMS error is at most $\tau/m$. Misses are the fraction of matrices
whose true RMS error exceeds $\tau$.}
\label{tab:target-margin}
\begin{tabular}{llrrrrr}
\toprule
& & \multicolumn{2}{c}{Mean depth} & \multicolumn{2}{c}{Misses (\%)} & \\
\cmidrule(lr){3-4}\cmidrule(lr){5-6}
$\tau$ & $m$ & Attn & MLP & Attn & MLP & Max error \\
\midrule
$10^{-1}$ & 1 & 7.91 & 4.23 & 1.9 & 0.0 & $0.124$ \\
 & 1.5 & 8.19 & 4.44 & 0.0 & 0.0 & $0.0780$ \\
 & 10 & 9.00 & 6.09 & 0.0 & 0.0 & $0.0567$ \\
 & RMS oracle & 7.83 & 4.02 & 0.0 & 0.0 & $0.0995$ \\
\midrule
$10^{-2}$ & 1 & 9.00 & 6.09 & 0.0 & 5.7 & $0.0567$ \\
 & 1.5 & 9.00 & 6.27 & 0.0 & 4.4 & $0.0567$ \\
 & 10 & 9.00 & 7.31 & 0.0 & 2.4 & $0.0389$ \\
 & RMS oracle & 8.99 & 4.85 & 0.0 & 0.0 & $0.00991$ \\
\midrule
$10^{-3}$ & 1 & 9.68 & 7.31 & 0.1 & 5.4 & $0.0389$ \\
 & 1.5 & 9.77 & 7.52 & 0.0 & 5.0 & $0.0384$ \\
 & 10 & 9.91 & 7.97 & 0.0 & 1.9 & $0.0326$ \\
 & RMS oracle & 9.55 & 5.36 & 0.0 & 0.0 & $0.000998$ \\
\bottomrule
\end{tabular}
\end{table}

\subsection{Reuse and per-type selection}
\label{app:reuse}

\Cref{tab:reuse} evaluates each saved matrix with a design endpoint chosen at an
earlier checkpoint, while recomputing its normalization. Reusing older choices
changes the error very little, and every variant stays far below fixed PE.

\Cref{tab:static-control} compares per-matrix selection with one fixed endpoint
per matrix type, chosen with exact SVD errors on early checkpoints, in the
spirit of AMO \citep{amo2026}. The per-type choice needs exact SVDs from the
same run and depends on which checkpoints are used to choose it. Spectral
estimation needs neither, has lower error at both depths, and gains most at
larger depth, where layers of the same type differ. Grouping would also help in
mixture-of-experts models with many expert matrices. Because the moments are
averages over the spectrum, averaging the measurements of a group gives the
moments of its pooled spectrum, so one fit can serve the whole group.

\begin{table}[!htbp]
\centering\small
\caption{\textbf{Reusing earlier choices.} Mean RMS error on checkpoints from
step $4$k. Columns give the checkpoint whose choice is used, namely the current one, the previous one, the first one,
or one updated only at every fourth checkpoint. Oracle is the minimax oracle.}
\label{tab:reuse}
\begin{tabular}{crrrrrrr}
\toprule
$K$ & Current & Previous & First & Every fourth & RMS oracle & Oracle & Fixed PE\\
\midrule
5 & $0.275$ & $0.275$ & $0.278$ & $0.275$ & $\best{0.273}$ & $0.381$ & $0.650$\\
8 & $0.0254$ & $0.0257$ & $0.0261$ & $0.0256$ & $\best{0.0246}$ & $0.0601$ & $0.160$\\
\bottomrule
\end{tabular}
\end{table}

\begin{table}[!htbp]
\centering\small
\caption{\textbf{Per-type versus per-matrix selection.} Mean RMS error on the
later checkpoints ($16$k--$30$k), with the per-type endpoints chosen on the earlier ones. Oracle
is the minimax oracle.}
\label{tab:static-control}
\begin{tabular}{crrrr}
\toprule
$K$ & Per type & Ours & RMS oracle & Oracle\\
\midrule
5 & $0.286$ & $0.283$ & $\best{0.282}$ & $0.383$\\
8 & $0.0290$ & $0.0273$ & $\best{0.0266}$ & $0.0639$\\
\bottomrule
\end{tabular}
\end{table}

\FloatBarrier
\section{Cost}
\label{app:pretraining-cost}
\label{app:cpu-cost}

For $M\in\R^{r\times n}$ with $r\le n$, one quintic iteration costs $2r^2n+r^3$
multiply--adds. The reductions of $A_j$ and $A_j^2$ add about $3r^2$, a fraction
$3/(2n+r)$ of an iteration, and use only matrices the iteration already holds.
The fit and the scoring work on a fixed number of measurements and a fixed
grid, so their cost does not depend on the matrix size, and they run on the CPU.
\Cref{tab:complexity} compares these costs with those of exact spectra.

\Cref{fig:gpu-cost,tab:frontier-shapes} report GPU times on one H200 for random
square and wide matrices and for matrix shapes of recent models trained with
Muon. The reductions take between $1$ and $3$~ms at every tested size, and the
fit takes $4$~ms of CPU time per matrix. The Newton--Schulz pass and the exact spectrum grow with the matrix.

The reductions and the fit run once every $T$ steps, and the
$K_{\mathrm{meas}}$-iteration measurement pass replaces one $K$-iteration pass
(\cref{eq:amortized}). For a model with $N$ hidden matrices, this adds
$Nt_{\rm red}/T$ of GPU time, $(K_{\mathrm{meas}}-K)/T$ iterations, and
$Nt_{\rm fit}/T$ of CPU time per step, where $t_{\rm red}$ and $t_{\rm fit}$
are the per-matrix times above. For our $1$B Muon model ($N=168$, $K=5$,
$T=200$, about $40$~s per step), this is under $1$~ms of reductions, $0.025$
iterations, and $4$~ms of CPU time per step.

\begin{table}[!htbp]
\centering\small
\caption{\textbf{Added cost of each spectrum estimator} for $M\in\R^{r\times n}$,
$r\le n$. One Newton--Schulz iteration costs $2r^2n+r^3$.}
\label{tab:complexity}
\begin{tabular}{llll}
\toprule
Method & Extra products & Work that grows with $M$ & Fixed work\\
\midrule
Ours & none & $\approx3r^2$ per iteration & fit, scoring\\
SVD (singular values) & -- & $O(r^2n)$ & --\\
Gram eigenvalues & $MM^\top$ & $O(r^2n+r^3)$ & --\\
\bottomrule
\end{tabular}
\end{table}

\begin{figure}[!htbp]
\centering
\includegraphics[width=\textwidth]{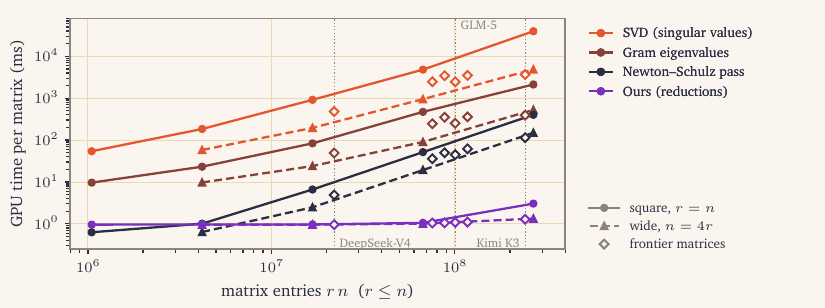}
\caption{\textbf{GPU time per matrix versus matrix size} (one H200, median).
Ours denotes the compiled reductions. Circles mark square matrices, triangles
wide matrices, and diamonds the matrices of \cref{tab:frontier-shapes}. The
Newton--Schulz pass runs in bf16, and the SVD and the Gram eigenvalues in fp32. Matrices are random Gaussian.}
\label{fig:gpu-cost}
\end{figure}

\begin{table}[!htbp]
\centering\small
\caption{\textbf{Matrix shapes of recent models trained with Muon}
\citep{deepseek2026v4,glm5team2026glm5,kimi2026k3} and GPU time per matrix in
milliseconds (one H200, median). The NS pass runs ten bf16
iterations, ours is the compiled reductions, and the last two columns compute
the eigenvalues of $MM^\top$ and the singular values. The DeepSeek-V4 attention output is the
second matrix of its grouped output projection.}
\label{tab:frontier-shapes}
\setlength{\tabcolsep}{3.5pt}
\begin{tabular}{lllrrrr}
\toprule
Model & Matrix & Shape & NS pass & Ours & Gram eig. & SVD\\
\midrule
DeepSeek-V4 & routed expert & $3072\times7168$ & $4.9$ & $\best{0.96}$ & $49$ & $480$\\
 & attention output & $7168\times16384$ & $62$ & $\best{1.1}$ & $358$ & $3481$\\
GLM-5 & attention output & $6144\times16384$ & $45$ & $\best{1.1}$ & $252$ & $2461$\\
 & dense FFN & $6144\times12288$ & $36$ & $\best{1.1}$ & $246$ & $2473$\\
Kimi K3 & attention output & $7168\times12288$ & $50$ & $\best{1.1}$ & $349$ & $3446$\\
 & dense FFN & $7168\times33792$ & $114$ & $\best{1.3}$ & $392$ & $3667$\\
\bottomrule
\end{tabular}
\end{table}

\FloatBarrier
\section{Inverse square roots}
\label{app:inverse}

For a positive-definite matrix $S$ scaled to have spectrum in $(0,1]$, consider
the coupled iteration \citep{higham1997stable,anil2020scalable}
\begin{equation}
 Y_0=S,\quad Z_0=I,\quad H_j=Z_jY_j,\quad Q_j=a_jI+b_jH_j,\quad
 Y_{j+1}=Y_jQ_j,\quad Z_{j+1}=Q_jZ_j,
 \label{eq:inverse-iteration}
\end{equation}
which computes $S^{-1/2}$ and is used in dense preconditioners, including the
scaled CANS construction of Online KL Shampoo \citep{cans2025,okls2026}.

\begin{proposition}[Scalar recursion of the coupled iteration]
\label{prop:inverse}
For an eigenvalue $\lambda>0$ of $S$, let $\zeta_j$ be the corresponding
eigenvalue of $Z_j$ and $z_j=\sqrt\lambda\,\zeta_j$. In exact arithmetic,
\begin{equation}
 z_0=\sqrt\lambda,
 \qquad
 z_{j+1}=a_jz_j+b_jz_j^3,
 \label{eq:inverse-correspondence}
\end{equation}
and the eigenvalue of $H_j$ is $z_j^2$. Thus $z_j$ follows the same scalar
recursion as in the polar case, and $z_j-1$ is the relative inverse-root error.
\end{proposition}

\begin{proof}
All iterates are polynomials in $S$ and commute, and induction gives
$Y_j=SZ_j$. The eigenvalue of $H_j=Z_jY_j$ is therefore $\lambda\zeta_j^2=z_j^2$,
so $\zeta_{j+1}=(a_j+b_jz_j^2)\zeta_j$, and multiplying by $\sqrt\lambda$ gives
\cref{eq:inverse-correspondence}.
\end{proof}

With $\gamma=\sqrt\lambda$, reductions of matrices the iteration already forms
measure the functions in \cref{tab:inverse-rows}, so the method of the polar
case carries over. For $0<\ell<1$, the odd cubic minimizing
$\max_{x\in[\ell,1]}|ax+bx^3-1|$ is
\begin{equation}
 h=1+\ell+\ell^2,
 \qquad
 \bar x=\sqrt{h/3},
 \qquad
 b=\frac{2}{1-h-(2h/3)\bar x},
 \qquad
 a=-bh,
 \label{eq:cubiccoeff}
\end{equation}
where $\bar x$ is the interior extremum of the error. This is the cubic iteration
behind scaled Newton--Schulz and CANS \citep{chen2014scaling,cans2025}, and it
defines our cubic candidates.

\begin{table}[!htbp]
\centering
\small
\caption{\textbf{Measurements from the coupled iteration.}
Each function is averaged over the distribution of $\gamma$.}
\label{tab:inverse-rows}
\begin{tabular}{ll}
\toprule
Reduction & Function\\
\midrule
$\tr H_j/n$, $\|H_j\|_F^2/n$
    & $z_j^2$, $z_j^4$\\
$\tr Y_j/n$, $\|Y_j\|_F^2/n$
    & $\gamma z_j$, $\gamma^2 z_j^2$\\
$\langle Y_j,H_j\rangle_F/n$, $\|Y_jH_j\|_F^2/n$
    & $\gamma z_j^3$, $\gamma^2 z_j^6$\\
$\tr Z_j/n$, $\|Z_j\|_F^2/n$
    & $z_j/\gamma$, $z_j^2/\gamma^2$\\
$\langle Z_j,H_j\rangle_F/n$, $\|H_jZ_j\|_F^2/n$
    & $z_j^3/\gamma$, $z_j^6/\gamma^2$\\
\bottomrule
\end{tabular}
\end{table}

We test the extension on regularized momentum Gram matrices
$S=\bar M\bar M^\top+\eps I$ from the primary run (\cref{tab:inverse-settings}).
\Cref{fig:inverse} shows that selecting a cubic routine from the fitted
spectrum improves on the fixed cubic routine at every depth, with the largest
gains on the more concentrated spectra.

\begin{figure}[!htbp]
\centering
\includegraphics[width=\textwidth]{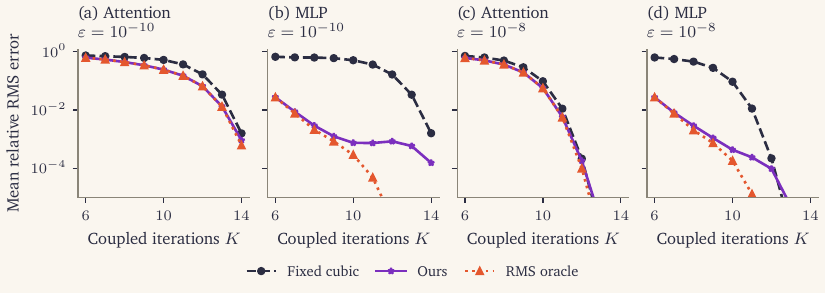}
\caption{\textbf{Inverse square roots.} Mean relative RMS error on regularized
momentum Gram matrices for the fixed cubic routine, spectral estimation, and the RMS
oracle over the same candidates.}
\label{fig:inverse}
\end{figure}

\begin{table}[!htbp]
\centering
\small
\caption{\textbf{Inverse-root test settings.}}
\label{tab:inverse-settings}
\begin{tabular}{ll}
\toprule
Component & Setting\\
\midrule
Test matrices
    & $\bar M\bar M^\top+\eps I$, $\eps\in\{10^{-10},10^{-8}\}$\\
Endpoints
    & $49$, log-spaced in $[10^{-5},10^{-1}]$\\
Candidate depths
    & $K=4,\ldots,18$\\
Grid
    & one low point, $150$ log-spaced\\
Entropy tolerance
    & $\kappa=3\times10^{-4}$\\
Measurement depth
    & $15$ ($\eps=10^{-10}$) or $13$ ($\eps=10^{-8}$)\\
\bottomrule
\end{tabular}
\end{table}

\FloatBarrier
\section{Pretraining details}
\label{app:training}

\subsection{Setup}
\label{app:pretraining-setup}

The model and optimizer follow \citet{okls2026}. \Cref{tab:pretraining-models}
lists the model, data, and training settings, and
\cref{tab:pretraining-optimizers} the hidden-matrix optimizer settings, which
are shared across model sizes. The hidden matrices of each layer (attention Q,
K, V, O and MLP up, gate, down) use Muon or OKLS, while embeddings, the output head,
normalization gains, and attention-gate projections use AdamW
\citep{loshchilov2019adamw}. The data order
is the same for every run. Each evaluation reads the next window of a held-out
shard, the same window for every run.

A hidden weight of shape $m\times n$ is initialized as
$W_{ij}\sim\mathcal N\bigl(0,[\sigma_*\sqrt{m/n}/(\sqrt m+\sqrt n)]^2\bigr)$,
embeddings as $\mathcal N(0,1)$, and the untied output head at zero, so hidden
gradients vanish at the first step. The learning rate warms up linearly, stays
constant, and decays linearly to zero over the final fraction $p_{\rm dec}$ of
training.

For Muon, with gradient $G_t\in\R^{m\times n}$,
\begin{align*}
 \mathcal B_t&=\beta \mathcal B_{t-1}+G_t, & N_t&=\beta \mathcal B_t+G_t,\\
 U_t&=\sqrt{m/n}\,\frac{\sqrt{\min(m,n)}}{\|\polar(N_t)\|_F}\polar(N_t), &
 \theta_{t+1}&=(1-\lambda_{\rm wd}\eta_t^2/\eta_{\rm peak})\theta_t-\eta_tU_t,
\end{align*}
where $\polar(N_t)$ is computed by the routine under study. Because the output
is rescaled to a fixed norm, routines differ only in the direction of the
update.

For OKLS, with the previous inverse-root factors $R_{\mathrm L},R_{\mathrm R}$,
\begin{align*}
 \mathcal B_t&=\beta_1\mathcal B_{t-1}+(1-\beta_1)G_t,
 &N_t&=(1-\beta_1)G_t+\beta_1\mathcal B_t,\\
 S_{\mathrm L}^+&=\beta_2S_{\mathrm L}+\tfrac{1-\beta_2}{n}(G_tR_{\mathrm R})(G_tR_{\mathrm R})^\top+\epsilon I,
 &S_{\mathrm R}^+&=\beta_2S_{\mathrm R}+\tfrac{1-\beta_2}{m}(R_{\mathrm L}G_t)^\top(R_{\mathrm L}G_t)+\epsilon I.
\end{align*}
The factors are symmetrized, their inverse roots $R_{\mathrm L}^+,R_{\mathrm R}^+$
are recomputed every step, and
\[
 \theta_{t+1}=(1-\lambda_{\rm wd}\eta_t^2/\eta_{\rm peak})\theta_t
 -\eta_t\chi R_{\mathrm L}^+N_tR_{\mathrm R}^+,
 \qquad
 \chi=\frac{\sqrt{m/n}}{\sqrt m+\sqrt n}
 \Bigl[\frac{1-\beta_1}{1+\beta_1}(1+2\beta_1-2\beta_1^3)\Bigr]^{-1/2}.
\]
The factors start from scaled $GG^\top$ and $G^\top G$ of the first nonzero
gradient. Compared with the OKLS setup, we hold the tokenizer, head dimension, MLP
expansion ratio, attention gate, initialization, and AdamW parameter group fixed
across model sizes, train for
slightly fewer steps, and use Polar Express instead of CANS for Muon.

\begin{table}[!htbp]
\centering\small
\caption{\textbf{Model and training settings.}}
\label{tab:pretraining-models}
\begin{tabular}{p{.30\textwidth}p{.19\textwidth}p{.19\textwidth}p{.19\textwidth}}
\toprule
Setting & $160$M & $300$M & $1$B\\
\midrule
Layers / width & $24$ / $768$ & $24$ / $1024$ & $24$ / $2048$\\
Query heads / KV heads & $12$ / $3$ & $16$ / $4$ & $32$ / $8$\\
Head dimension / MLP width & $64$ / $2304$ & $64$ / $3072$ & $64$ / $6144$\\
Non-embedding / total params & $163.1$M / $262.9$M & $289.9$M / $423.0$M & $1159.3$M / $1425.6$M\\
Attention / MLP & \multicolumn{3}{p{.63\textwidth}}{grouped-query attention with a per-head sigmoid gate / SwiGLU}\\
Normalization / positions & \multicolumn{3}{p{.63\textwidth}}{pre-RMSNorm and final RMSNorm ($\epsilon=10^{-6}$) / RoPE ($\theta=10^4$)}\\
Data & \multicolumn{3}{p{.63\textwidth}}{Nemotron-CC v2 High-Quality \citep{nvidia2025nemotronccv2,nvidia2025nemotronnano2}}\\
Tokenizer / vocabulary & \multicolumn{3}{p{.63\textwidth}}{\texttt{tiiuae/falcon-7b} / $65{,}024$}\\
Context / sequences per step & \multicolumn{3}{p{.63\textwidth}}{$4096$ / $1024$ ($4{,}194{,}304$ tokens)}\\
Steps / warmup & \multicolumn{3}{p{.63\textwidth}}{$4769$ ($20$B tokens) / $250$}\\
Validation & \multicolumn{3}{p{.63\textwidth}}{$4{,}194{,}304$ held-out tokens every $250$ steps}\\
AdamW (non-hidden params) & \multicolumn{3}{p{.63\textwidth}}{lr $3\times10^{-4}$, $(\beta_1,\beta_2)=(0.9,0.95)$, $\epsilon=10^{-8}$, no weight decay}\\
Precision & \multicolumn{3}{p{.63\textwidth}}{bf16 autocast, fp32 weights and optimizer state}\\
Clipping / dropout / z-loss & \multicolumn{3}{p{.63\textwidth}}{none}\\

\bottomrule
\end{tabular}
\end{table}

\begin{table}[!htbp]
\centering\small
\caption{\textbf{Hidden-matrix optimizer settings.} The decay coefficient
$\lambda_{\rm wd}$ enters as $\theta\leftarrow(1-\lambda_{\rm wd}\eta_t^2/\eta_{\rm peak})\theta$.}
\label{tab:pretraining-optimizers}
\begin{tabular}{lrr}
\toprule
Setting & Muon & OKLS\\
\midrule
Peak learning rate & $0.01202$ & $0.09434$\\
Momentum $\beta$ or $\beta_1$ & $0.9586$ & $0.9684$\\
Factor decay $\beta_2$ & -- & $0.9482$\\
Factor diagonal shift $\epsilon$ & -- & $10^{-9}$\\
Weight decay $\lambda_{\rm wd}$ & $1.008\times10^{-4}$ & $0.0303$\\
Initialization scale $\sigma_*$ & $0.1877$ & $0.07539$\\
Decay fraction $p_{\rm dec}$ & $0.7677$ & $0.7319$\\
Matrix-function precision & bf16 & fp32\\
\bottomrule
\end{tabular}
\end{table}

\subsection{Matrix-function routines}
\label{app:pretraining-method}

\paragraph{Muon.}
The routines use the normalization, candidates, and fitting of
\cref{app:scaling,app:dictionary} with the settings in
\cref{tab:polar-settings,tab:pretraining-routines}. The fixed baseline is $\PE[10^{-5},1]$
at the run's depth. The oracle recomputes the minimax choice from an exact SVD every $500$ steps and uses the fixed baseline until its first choice. Spectral
estimation also starts from the fixed baseline. Every $T$ steps it runs the
measurement pass, which
also produces that step's update, fits the moments on the CPU, and uses the
selected routine until the next measurement. If a fit fails, the previous
routine is kept. Our runs also record exact spectra at each measurement to measure
selection quality (\cref{app:pretraining-validation}), and these records never
enter selection.

\paragraph{OKLS.}
The inverse roots use the coupled iteration of \cref{app:inverse}, scaled by an
upper bound on the largest eigenvalue, in fp32 because the coupled iteration
drifts in bf16. The measurements include all traces, squared norms, and inner
products of $H_j$, $Y_j$, $Z_j$ and of the products the iteration already forms
(\cref{tab:inverse-rows}), and those involving $Z_j$ add information about small
eigenvalues. The fixed baseline uses the cubic routine designed for eigenvalues
above $10^{-7}\lambda_{\max}$. An exact eigendecomposition every $500$ steps
only updates $\lambda_{\max}$, which, like the Frobenius norm in PE, bounds the scaled spectrum by one, while the relative lower endpoint stays at $10^{-7}$.

\paragraph{Comparisons.}
\label{app:pretraining-period}
We use $T=200$ throughout, and \cref{app:pretraining-extra} also reports $T=500$.
 \Cref{fig:pretraining}
averages the training-loss gaps over short windows.

\begin{table}[!htbp]
\centering\small
\caption{\textbf{Pretraining routine settings.}}
\label{tab:pretraining-routines}
\begin{tabular}{p{.2\textwidth}p{.33\textwidth}p{.37\textwidth}}
\toprule
& Muon & OKLS\\
\midrule
Scoring cutoff & $10^{-5}$ & $10^{-7}$ (relative)\\
Candidates & $49$ PE endpoints in $[10^{-5},10^{-1}]$ at the run's depth $K$ & cubic, $12$ endpoints per decade in $[10^{-7},1]$, at the run's depth $K$\\
Measurement pass & $10$ iterations, $41$ measurements & $10$ iterations, $94$ measurements\\
Measurement period $T$ & $200$ ($500$ also at $160$M and $300$M) & $200$ ($500$ also at $160$M and $300$M)\\
Fixed baseline & $\PE[10^{-5},1]$ & cubic routine for eigenvalues above $10^{-7}\lambda_{\max}$\\
Oracle & \multicolumn{2}{p{.72\textwidth}}{minimax choice from the exact spectrum every $500$ steps}\\
Solver & \multicolumn{2}{p{.72\textwidth}}{Clarabel in float64 on CPU, tolerance $10^{-8}$, one retry at $10^{-11}$}\\
\bottomrule
\end{tabular}
\end{table}

\subsection{Additional pretraining results}
\label{app:pretraining-extra}

\paragraph{Measurement period.}
\Cref{tab:pretraining-period} compares $T=200$, used throughout, with $T=500$,
which we also ran at $160$M and $300$M. Both periods improve on the fixed baseline in
every configuration, and neither is better in all of them.

\begin{table}[!htbp]
\centering\small
\caption{\textbf{Measurement period.} Final validation loss of our method with
$T=200$ and $T=500$.}
\label{tab:pretraining-period}
\begin{tabular}{llcrrr}
\toprule
Optimizer & Size & $K$ & Fixed & $T=200$ & $T=500$\\
\midrule
Muon & 160M & $5$ & $2.761$ & $2.751$ & $\best{2.748}$\\
Muon & 300M & $5$ & $2.662$ & $2.655$ & $\best{2.651}$\\
OKLS & 160M & $5$ & $2.687$ & $2.673$ & $\best{2.671}$\\
OKLS & 300M & $5$ & $2.597$ & $\best{2.590}$ & $2.591$\\
\bottomrule
\end{tabular}
\end{table}

\paragraph{Practitioner endpoint.}
The default Polar Express routine \citep{amsel2026polar} is designed for
$[10^{-3},1]$, a narrower interval than the $[10^{-5},1]$ of our fixed baseline.
\Cref{tab:pretraining-practitioner} compares both fixed routines with our
method at $K=5$. The less conservative endpoint improves slightly on the
fixed baseline at $160$M and matches it at $300$M, and per-matrix selection improves on
both.

\begin{table}[!htbp]
\centering\small
\caption{\textbf{Practitioner endpoint.} Muon, $K=5$, final validation loss.}
\label{tab:pretraining-practitioner}
\begin{tabular}{lrrr}
\toprule
Size & $\PE[10^{-5},1]$ & $\PE[10^{-3},1]$ & Ours\\
 & (baseline) & (default) & \\
\midrule
160M & $2.761$ & $2.765$ & $\best{2.751}$\\
300M & $2.662$ & $2.669$ & $\best{2.655}$\\
\bottomrule
\end{tabular}
\end{table}

\subsection{Selection quality during training}
\label{app:pretraining-validation}

At every measurement we also recorded exact spectra and compared the chosen routine
with the RMS oracle on the same spectra (\cref{tab:pretraining-selection}). The error of spectral
estimation is close to the oracle's and far below the fixed baseline's. \Cref{fig:pretraining-selection}
breaks this down by matrix type. On attention Q and O, the minimax rule picks
much smaller endpoints than the RMS rule, because it protects the worst
direction. On the other matrix types spectral estimation picks smaller, more
conservative endpoints than the RMS oracle.

\Cref{tab:pretraining-offline} repeats the selection on saved training matrices
and varies two things separately. Across columns, measured moments, in bf16 for
Muon and fp32 for OKLS, give nearly the same selection quality as exact ones.
Across rows, the OKLS reductions involving $Y_j$ and $Z_j$ matter, since traces
and norms of $H_j$ alone give clearly worse choices.

\begin{table}[!htbp]
\centering\small
\caption{\textbf{Selection quality during training.} Agreement is the share of
measurements where spectral estimation picks the RMS oracle's endpoint,
distance the median gap between the two endpoints in decades, regret our RMS
error minus the RMS oracle's, and $N$ the number of matrix or factor
measurements. $160$M runs.}
\label{tab:pretraining-selection}
\setlength{\tabcolsep}{3.3pt}
\begin{tabular}{lrrrrrrrrr}
\toprule
& & & & \multicolumn{2}{c}{Endpoint} & \multicolumn{4}{c}{RMS error}\\
\cmidrule(lr){5-6}\cmidrule(lr){7-10}
Optimizer & $K$ & $T$ & $N$ & Agreement (\%) & Distance & Fixed & RMS oracle & Ours & Regret\\
\midrule
Muon & 4 & 200 & 3,864 & 58.6 & 0.00 & 0.6681 & 0.2226 & 0.2254 & 0.0028\\
Muon & 4 & 500 & 1,512 & 61.1 & 0.00 & 0.6603 & 0.2194 & 0.2219 & 0.0025\\
Muon & 5 & 200 & 3,864 & 27.3 & 0.25 & 0.6359 & 0.0967 & 0.1066 & 0.0099\\
Muon & 5 & 500 & 1,512 & 28.2 & 0.25 & 0.6322 & 0.0724 & 0.0828 & 0.0104\\
OKLS & 5 & 200 & 4,704 & 77.8 & 0.00 & 0.6267 & 0.2390 & 0.2394 & 0.0004\\
OKLS & 5 & 500 & 2,016 & 73.8 & 0.00 & 0.6149 & 0.2226 & 0.2232 & 0.0006\\
\bottomrule
\end{tabular}
\end{table}

\begin{figure}[!htbp]
\centering
\includegraphics[width=\textwidth]{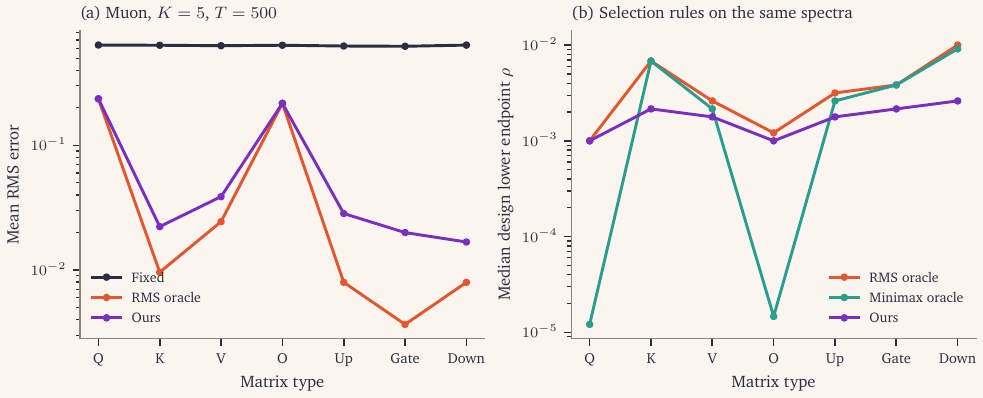}
\caption{\textbf{Selection by matrix type} (Muon $160$M, $K=5$, $T=500$).
\textbf{(a)} Mean RMS error on the recorded exact spectra. \textbf{(b)} Median
design endpoint chosen by spectral estimation and by both oracles on the same
spectra.}
\label{fig:pretraining-selection}
\end{figure}

\begin{table}[!htbp]
\centering\small
\caption{\textbf{Moment precision and measurement set.} Mean regret against
the RMS oracle on saved training matrices. Columns compare moments computed exactly
with moments measured during the pass (bf16 for Muon, fp32 for OKLS), and rows
the reductions used. The exact columns cover $504$ Muon matrices and $1008$
OKLS factors, and the measured ones $504$ and $56$.}
\label{tab:pretraining-offline}
\begin{tabular}{llrrr}
\toprule
Optimizer & Reductions & Count & Exact & Measured\\
\midrule
Muon & all & $41$ & $0.0078$ & $0.0090$\\
OKLS & all & $94$ & $0.0004$ & $0.0008$\\
OKLS & $H_j$ only & $20$ & $0.0019$ & $0.0044$\\
\bottomrule
\end{tabular}
\end{table}

\FloatBarrier
\section{Related work}
\label{app:related}

\paragraph{Polynomial iterations for matrix functions.}
Newton--Schulz iterations compute the polar factor using only matrix products
\citep{bjorck1971iterative,kovarik1970,higham2008functions}, and coupled
versions compute inverse square roots \citep{higham1997stable}. Later work
improves the polynomials. Scaled Newton--Schulz speeds up convergence
\citep{chen2014scaling}, Polar Express designs each iteration by minimax
approximation on a spectral interval \citep{amsel2026polar}, and CANS uses
Chebyshev-type designs \citep{cans2025}. These routines need the spectral
interval in advance and typically use the same one for every matrix. Our method
keeps these polynomial families and chooses the interval and depth from the
spectrum of the matrices being processed.

Other work makes each orthogonalization cheaper. Gram Newton--Schulz iterates
on the smaller Gram matrix with symmetric products \citep{zhang2026gram}, Dion and Dion3 reduce communication in distributed training
\citep{ahn2025dion,amsel2026dion3}, MuonBP orthogonalizes blocks between full
steps \citep{khaled2026muonbp}, and CacheMuon reuses work from earlier steps
\citep{dev2026cachemuon}. These methods lower the cost of an iteration or of a step, whereas spectral estimation chooses which iterations to run and can be
combined with them.

\paragraph{Adapting the iteration to the matrix.}
ROOT tailors coefficients to the matrix
dimensions \citep{he2025root}. AMO observes spectra early in training and then
fixes one routine per matrix type \citep{amo2026}, and \cref{app:reuse} compares
per-matrix selection with this kind of per-type choice. PRISM fits the
polynomial to the current spectrum at every iteration using randomized
sketches \citep{prism2026}. Unlike these works, we read the spectral information from products that
Newton--Schulz already forms, without random probes, and use it to choose a
complete routine for later steps.

\paragraph{Inexact orthogonalization.}
Analyses of Muon with inexact orthogonalization relate the approximation error
to convergence and to the choice of learning rate and momentum
\citep{shulgin2026inexact,kim2026convergence,choudhury2026nesterov,qian2026degenerate,do2026muon}.
Spectral estimation reduces this error at a fixed number of iterations.

\paragraph{Spectral density estimation from moments.}
Estimating a spectral density from traces of matrix polynomials is standard in
numerical linear algebra. The kernel polynomial method \citep{weisse2006kpm},
stochastic Lanczos quadrature \citep{ubaru2017slq,chen2021slq}, and related
methods \citep{lin2016density} estimate these traces with random probe vectors,
and maximum-entropy fits turn a few moments into a density
\citep{mead1984,silver1997entropy}. \citet{fan2020spectrum} use such estimates
to adapt polynomial approximations to the spectrum, which is close in spirit
to our selection step. Finitely many moments do not identify a distribution
\citep{krein1977}, but they determine averages of functions that the measured
polynomials approximate well \citep{braverman2022sublinear}, and our prediction
theorem applies this to the moments that Newton--Schulz produces. Here the moments are traces of matrices the iteration already forms, so they
need neither probe vectors nor extra matrix products.